\documentclass[11pt]{article}

\usepackage[a4paper,margin=1in]{geometry}
\usepackage[T1]{fontenc}
\usepackage{lmodern}
\usepackage{microtype}
\newcommand{\Attnk}{\mathrm{Attn}_k}
\newcommand{\Attn}{\operatorname{Attn}}
\newcommand{\Var}{\operatorname{Var}}

\usepackage{amsmath,amssymb,amsthm,mathtools,bm}
\usepackage[margin=1in]{geometry}
\usepackage[colorlinks=true,linkcolor=blue]{hyperref}
\usepackage[capitalize,nameinlink]{cleveref}

\usepackage{amsmath,amssymb,amsthm,mathtools,bm}
\newcommand{\EE}{\mathbb{E}}

\usepackage{booktabs}
\usepackage{hyperref}
\hypersetup{
  colorlinks=true,
  linkcolor=blue,
  citecolor=blue,
  urlcolor=blue
}

\numberwithin{equation}{section}
\newtheorem{theorem}{Theorem}[section]
\newtheorem{lemma}[theorem]{Lemma}
\newtheorem{proposition}[theorem]{Proposition}
\newtheorem{corollary}[theorem]{Corollary}
\theoremstyle{remark}
\newtheorem{remark}[theorem]{Remark}

\DeclareMathOperator{\softmax}{softmax}
\newcommand{\TV}{\operatorname{TV}}
\newcommand{\KL}{\operatorname{KL}}
\newcommand{\E}{\mathbb{E}}

\newcommand{\R}{\mathbb{R}}

\usepackage{booktabs}
\usepackage{siunitx}      
\usepackage{threeparttable}
\usepackage{rotating}     
\usepackage{graphicx}
\graphicspath{{figs/}}             
\DeclareGraphicsExtensions{.pdf,.png,.jpg}  

\usepackage{graphicx}
\graphicspath{{TopK_Attention_Figures/}}   
\DeclareGraphicsExtensions{.pdf,.png,.jpg}

\usepackage{algorithm}
\usepackage{algpseudocode}
\algrenewcommand\algorithmicrequire{\textbf{Input:}}
\algrenewcommand\algorithmicensure{\textbf{Output:}}
\algnewcommand{\LineComment}[1]{\State \(\triangleright\) #1}
\usepackage{booktabs}
\usepackage{caption}
\usepackage{float}      
\usepackage[section]{placeins} 

\usepackage{tabularx}
\usepackage{booktabs}

\title{A Mathematical Theory of Top-\emph{k} Sparse Attention via Total Variation Distance}

\author{
  Georgios Tzachristas\textsuperscript{1,2},
  Lei Deng\textsuperscript{1},
  Ioannis Tzachristas\textsuperscript{3},
  Gong Zhang\textsuperscript{1},
  Renhai Chen\textsuperscript{1}
}

\date{
\textsuperscript{1}\,Theory Laboratory, Central Research Institute, 2012 Laboratory, \\ Huawei Technologies Co., Ltd., Shenzhen, China.\\
\texttt{\{tzachristasgeorgios, deng.lei2, nicholas.zhang, chenrenhai\}@huawei.com}\\[0.75em]
\textsuperscript{2}\,National Technical University of Athens, Greece\\
\texttt{el23102@mail.ntua.gr}\\[1em]
\textsuperscript{3}\,Trustworthy Technology and Engineering Laboratory, Heisenberg Research Center, 
\newline
2012 Laboratory, Huawei Technologies Co., Ltd., Munich, Germany.\\
\texttt{ioannis.tzachristas@huawei.com}\\[1em]
}

\begin{document}

\maketitle
\begin{abstract}
We develop a unified mathematical framework for certified Top-$k$ attention truncation that quantifies approximation error at both the distribution and output levels. For a single attention distribution $P$ and its Top-$k$ truncation $\hat P$, we show that the total-variation distance coincides with the discarded softmax tail mass and satisfies the exact identity
\[
\TV(P,\hat P) = 1 - e^{-\KL(\hat P \Vert P)},
\]
replacing generic inequalities by a sharp equality specific to Top-$k$ truncation. Building on this, we derive a hierarchy of non-asymptotic deterministic bounds---from a single boundary gap, through multi-gap and blockwise variants---that control $\TV(P,\hat P)$ using only the ordered logits.

We then turn to the attention outputs. Using an exact head--tail decomposition, we prove that the vector error factorizes as
\[
\|\Attn(q,K,V) - \Attnk(q,K,V)\|_2
= \tau \,\|\mu_{\text{tail}} - \mu_{\text{head}}\|_2,
\qquad \tau = \TV(P,\hat P),
\]
yielding a new head--tail diameter bound
\[
\|\Attn(q,K,V) - \Attnk(q,K,V)\|_2 \le \tau\,\mathrm{diam}_{H,T},
\]
where $\mathrm{diam}_{H,T}$ is the cross-set diameter of the value vectors. This introduces the value matrix $V$ as an explicit geometric parameter in the error bound, and we further obtain variance-based refinements that relate the output error to $\Var_P(V)$.

Under an i.i.d.\ Gaussian score model $s_i \sim \mathcal N(\mu,\sigma^2)$, we derive closed-form expressions for the softmax tail mass and show that the minimal $k_\varepsilon$ ensuring $\TV(P,\hat P) \le \varepsilon$ obeys an asymptotic design rule $k_\varepsilon/n \approx \Phi_c(\sigma + \Phi^{-1}(\varepsilon))$. These analytical results motivate two certified selection algorithms,
$\Delta_k$-Search and MC-Search, which adaptively score only a subset of
keys and stop as soon as they can certify that a prescribed Top-$k$
truncation satisfies $\TV(P,\widehat P)\le\varepsilon$.

Empirical evaluations on \texttt{bert-base-uncased} attention maps and synthetic long-context logits confirm the predicted scaling of the certified sparsity ratio $k_\varepsilon/n$, reveal heavier-tailed logits than the Gaussian idealization, and demonstrate that certified Top-$k$ can reduce scored keys by a factor of $2$--$4\times$ on average---and by orders of magnitude on sharply peaked heads---while rigorously respecting the prescribed total-variation error.
\end{abstract}

\vspace{1em}

\section{Introduction}
\label{sec:intro}

The attention mechanism is the computational core of modern transformer architectures.
For a sequence of $n$ tokens with query, key, and value representations
$Q,K\in\R^{n\times d}$ and $V\in\R^{n\times d_v}$,
the standard attention operator
\[
\Attn(Q,K,V)=\softmax\!\left(\frac{QK^\top}{\sqrt d}\right)V
\]
requires $O(n^2 d)$ time and $O(n^2)$ memory to compute, dominated by the
pairwise dot products between all queries and keys.
As the context length $n$ grows into the tens or hundreds of thousands,
this quadratic scaling has become the primary bottleneck in large language models (LLMs).

A natural remedy is to exploit the observation that, for each query, only a small number of keys contribute significantly to the attention output. Top-$k$ sparse attention methods approximate the softmax weights by, for each query $q_i$, retaining only the $k$ largest scores $q_i^\top k_j$ over $j$ and zeroing the rest. Such sparsity can lead to subquadratic complexity and aligns well with modern hardware, and numerous practical variants have been proposed, including block-based, hierarchical, and hash-based sparse attention. Several mainstream open-source systems now adopt Top-$k$ sparse attention in their inference stacks.

Despite impressive empirical performance, however, the mathematical foundations of these methods remain poorly understood. In this work we study Top-$k$ attention truncation through the lens of probability and optimization. Fundamental questions persist:
\begin{itemize}
    \item How does truncating the softmax affect the output distribution and the resulting attention output?
    \item What deterministic or probabilistic guarantees can bound the error introduced by keeping only the top $k$ keys?
    \item Can we design algorithms that certify in advance that a given truncation achieves a desired accuracy?
\end{itemize}

\subsection{Contributions}

This paper develops a unified mathematical theory of Top-$k$ attention that quantifies the
effect of truncation at both the distribution and output levels, and connects these guarantees
to certified sparse-attention algorithms. Our main contributions are:

\begin{enumerate}
  \item \textbf{Distribution-level theory for Top-$k$ truncation.}
  For a softmax distribution $P$ and its Top-$k$ truncation $\widehat P$, we show that the
  total-variation distance coincides with the discarded softmax tail mass and satisfies the
  exact identity
  \[
    \TV(P,\widehat P) \;=\; 1 - e^{-\KL(\widehat P \,\Vert\, P)}.
  \]
  This replaces generic inequality-based bounds with a sharp equality specific to Top-$k$
  truncation and serves as the foundation for all subsequent certificates.

  \item \textbf{Deterministic certificates from score gaps and blocks.}
  Using only the ordered logits $s_1 \ge \cdots \ge s_n$, we derive non-asymptotic upper
  bounds on $\TV(P,\widehat P)$ in terms of:
  (i) a single boundary gap $s_k - s_{k+1}$;
  (ii) multi-gap (pigeonhole) refinements; and
  (iii) blockwise extensions that operate on groups of consecutive keys.
  These results yield a hierarchy of deterministic certificates, from fine-grained to
  coarse-grained, without any distributional assumptions.

  \item \textbf{Output-level error bounds and geometry of the value matrix.}
  We establish an exact head--tail decomposition
  \[
    \bigl\|\Attn(q,K,V) - \Attn_k(q,K,V)\bigr\|_2
      \;=\; \tau \,\bigl\|\mu_{\text{tail}} - \mu_{\text{head}}\bigr\|_2,
      \qquad \tau = \TV(P,\widehat P),
  \]
  and obtain a new head--tail diameter bound
  \[
    \bigl\|\Attn(q,K,V) - \Attn_k(q,K,V)\bigr\|_2
      \;\le\; \tau\,\mathrm{diam}_{H,T},
  \]
  where $\mathrm{diam}_{H,T}$ is the cross-set diameter of the value vectors.
  This introduces the value matrix $V$ as an explicit geometric parameter in the error and
  is complemented by variance-based bounds that relate the output error to $\Var_P(V)$.
  We further give a graph-theoretic characterization of the optimal head--tail partition:
  the "minimax" cross diameter is achieved by cutting the lightest edge of a maximum spanning
  tree.

  \item \textbf{Probabilistic laws under a Gaussian score model.}
  Under an i.i.d.\ Gaussian model $s_i \sim \mathcal N(\mu,\sigma^2)$, we derive closed-form
  expressions for the softmax tail mass and show that the minimal $k_\varepsilon$ ensuring
  $\TV(P,\widehat P)\le\varepsilon$ obeys an asymptotic design rule
  \[
    k_\varepsilon/n \;\approx\; \Phi_c\bigl(\sigma + \Phi^{-1}(\varepsilon)\bigr).
  \]
  This links error tolerance, logit variance, and expected sparsity, providing a principled
  guideline for choosing $k$ in high-dimensional regimes.

  \item \textbf{Certified selection algorithms and empirical validation.}
  We translate these analytical results into two certified Top-$k$ selection procedures:
  \emph{$\Delta_k$-Search}, which uses a gap-based certificate, and \emph{MC-Search},
  which uses mass certificates and their blockwise extensions.
  Both algorithms adaptively score only as many keys as needed to certify
$\TV(P,\widehat P)\le\varepsilon$, with $\Delta_k$-Search verifying a
prescribed Top-$k$ and MC-Search additionally returning the minimal
mass-certified size $k_\varepsilon(q)$ for each query.

  Experiments on \texttt{bert-base-uncased} attention maps and synthetic long-context logits
  confirm the predicted scaling of the certified sparsity ratio $k_\varepsilon/n$, reveal
  heavier-tailed logits than the Gaussian idealization, and demonstrate that certified
  Top-$k$ can reduce the number of scored keys by factors of $2$--$4\times$ on average,
  and by orders of magnitude for sharply peaked heads, while rigorously respecting the
  prescribed TV budget.
\end{enumerate}

\subsection{Applications to LLM-based systems}
\label{subsec:llm-applications}

The theory developed here enables \textbf{provably efficient attention} in practical LLM deployments. By coupling the \textbf{exact tail-mass identity}
\[
\mathrm{TV}(P,\hat P) = 1 - e^{-\mathrm{KL}(\hat P \,\|\, P)}
\]
with our \textbf{certified selectors} ($\Delta_k$-Search and MC-Search), one can adaptively choose the minimal Top-$k$ per head and query that meets a prescribed tolerance $\varepsilon$, potentially yielding \textbf{subquadratic complexity} in the sequence length with explicit accuracy guarantees (Sections~\ref{sec:dist-theory}--\ref{sec:algorithms}). The same certificates extend to structured, blockwise pruning across heads, windows, or tiles, supporting head/window selection under \textbf{verifiable error budgets} (Sections~\ref{sec:dist-theory} and~\ref{sec:algorithms}). Because the certification is orthogonal to implementation, it composes with fast attention kernels to reduce FLOPs without altering semantics, and it can inform serving-time trade-offs (latency/throughput vs.\ bounded deviation) as well as training-time objectives that encourage larger score gaps or target TV budgets. Empirically, certified sparsity scales predictably with context length while preserving the requested bound, making the approach suitable for \textbf{long-context inference} and resource-constrained deployment (Sections~\ref{sec:algorithms}--\ref{sec:real-attn-validation}).

\subsection{Scope and Structure}

Our analysis focuses on single-head attention with fixed queries,
but the results extend naturally to multi-head settings and to other exponential-family attention mechanisms.
The paper is organized as follows.
Section~3 introduces notation and basic definitions.
Section~4 derives fundamental identities connecting total variation, KL divergence, and related quantities, and establishes deterministic gap and blockwise bounds.
Section~5 converts these distribution-level guarantees into output-level
error guarantees, including geometric and graph-theoretic characterizations of the value matrix.
Section~6 develops probabilistic predictions under a Gaussian score model.
Section~7 presents certified selection algorithms translating these bounds into practice.
Section~8 empirically validates the theory on real attention maps and synthetic long-context logits.
Section~9 concludes with a discussion of implications, limitations, and future directions.

In summary, this work provides a rigorous mathematical foundation for Top-$k$ attention,
showing that its empirical robustness follows from the precise exponential decay of
the discarded softmax mass and from structural properties of the value vectors,
and that these insights enable the design of attention mechanisms with certified accuracy guarantees.

\section{Related Work}

The quadratic time and memory complexity of the standard attention operator has motivated extensive research into efficient and sparse alternatives. A full attention layer requires $O(n^2 d)$ computation, which becomes prohibitive for long sequences. Consequently, a wide variety of methods have sought to approximate or restrict the attention pattern while preserving model quality. Most of these approaches rely on heuristic sparsity or low-rank assumptions, and even when they come with theoretical guarantees, these are usually stated in terms of kernel or matrix approximation or model expressivity rather than per-query accuracy of the attention distribution. Our work differs in offering a unified mathematical analysis of \textit{Top-$k$} truncation with provable, certified error bounds on the resulting softmax distribution.

\paragraph{Fixed and Learned Sparse Patterns.}
Early efforts such as the \textbf{Sparse Transformer}~\cite{child2019sparse} introduced fixed blockwise and strided attention masks that reduce complexity while maintaining global context. Subsequent models, including \textbf{Longformer}~\cite{beltagy2020longformer} and \textbf{BigBird}~\cite{zaheer2020bigbird}, combined local sliding windows with a small number of global tokens to extend attention to thousands of tokens. These methods improve scalability and can be trained effectively, but they impose hard-coded sparsity patterns rather than deriving them from the attention statistics themselves. They do not, however, provide explicit per-query deterministic or probabilistic bounds on the approximation error introduced by truncating a given attention distribution.

\paragraph{Low-Rank and Kernel Approximations.}
A second line of work compresses the attention matrix through low-rank or kernel-based approximations. \textbf{Linformer}~\cite{wang2020linformer} projects the key and value matrices into a lower-dimensional subspace, achieving linear complexity in sequence length. \textbf{Performer}~\cite{choromanski2021performer} and \textbf{Random Feature Attention}~\cite{peng2021random} replace the softmax kernel by positive random feature maps, while \textbf{Nyströmformer}~\cite{xiong2021nystromformer} approximates the attention matrix using Nyström sampling. These techniques provide substantial speedups and are accompanied by theoretical analyses in terms of kernel or matrix approximation error, or unbiasedness and concentration of the underlying estimators. However, their guarantees typically concern approximation of the attention kernel or matrix, rather than explicit per-query bounds on the deviation between the approximate and exact softmax distributions that can be turned into certification rules. In contrast, our analysis derives \textit{closed-form deterministic} and \textit{probabilistic} guarantees on the deviation between the truncated and exact softmax distributions.

\paragraph{Memory-Efficient Implementations.}
Complementary work has focused on optimizing the \textit{exact} softmax computation for modern hardware. \textbf{FlashAttention}~\cite{dao2022flashattention} reorganizes the computation to minimize GPU memory access, achieving large throughput gains without altering the mathematical definition of attention. While such systems-level improvements are orthogonal to sparsification, our certified truncation algorithms can be combined with such implementations to yield both theoretical and practical efficiency.

\paragraph{Analytical and Theoretical Perspectives.}
Several studies have examined the softmax function from a numerical or probabilistic viewpoint. Classical works on exponential families and on regular variation~\cite{resnick1987extreme,pitman2006combinatorial} provide the probabilistic foundation for our Gaussian score model analysis. Despite these advances, prior work stops short of delivering \textit{exact identities} or \textit{certifiable algorithms} that bound, for a given finite score set, the total-variation error arising from truncating the attention distribution.

\paragraph{Position of This Work.}
Our contribution is orthogonal and complementary to these lines of research. We develop, to our knowledge, the first rigorous closed-form TV/KL-based mathematical theory of \textit{Top-$k$} attention truncation, proving an exact identity between total-variation distance and KL divergence, deriving deterministic and Gaussian-model bounds, and translating them into \emph{$\Delta_k$-Search} and \emph{MC-Search} algorithms that certify a desired accuracy ($\mathrm{TV}\le\varepsilon$) before computing all dot products. Whereas most existing methods trade accuracy for speed heuristically or provide only kernel-level approximation guarantees, our framework provides \textit{verifiable per-query guarantees} linking sparsity, score variance, and approximation error—establishing a principled foundation for provably correct sparse attention.

\paragraph{Comparative Guarantees.}
Most existing efficient-attention architectures---including Linformer~\cite{wang2020linformer}, Performer~\cite{choromanski2021performer}, Reformer~\cite{kitaev2020reformer}, and other kernelized or low-rank variants---reduce computational cost by projecting or approximating the attention kernel. Their analyses typically provide bounds on kernel or matrix approximation error, or characterize unbiasedness and concentration of random feature estimators, rather than giving explicit per-query bounds (for example, in total-variation distance) on the deviation between the approximate and true softmax distributions that can be used as certification rules for a particular score vector.

In contrast, our framework establishes \emph{finite-sample, non-asymptotic guarantees} in total-variation distance. 
The exact identity
\[
\mathrm{TV}(P,\widehat P) = 1 - e^{-\mathrm{KL}(\widehat P\|P)}
\]
for the pair $(P,\widehat P)$ arising from Top-$k$ truncation links the truncation error directly to discarded softmax mass, yielding deterministic certificates that hold for any finite score set.
Unlike prior methods, our certified algorithms (\emph{$\Delta_k$-Search} and \emph{MC-Search}) guarantee that the truncated attention achieves $\mathrm{TV}\le\varepsilon$ before any approximation is accepted.
Thus, the proposed theory replaces heuristic sparsity with \textbf{verifiable correctness}, complementing empirically efficient methods with provable accuracy control.

\section{Preliminaries and Notation}
\label{sec:prelim}

We formalize the notation and definitions used throughout the paper and record several basic properties of the softmax and its Top-$k$ truncation.

\subsection{Attention Scores and Softmax Probabilities}

For a single query vector $q\in\R^{d}$, keys $k_1,\dots,k_n\in\R^{d}$, and values $v_1,\dots,v_n\in\R^{d_v}$,
define the (scaled) similarity scores
\[
s_i = \frac{q^\top k_i}{\sqrt d}, \qquad i=1,\dots,n.
\]

\paragraph{Indexing convention.}
For each query $q$, let $\pi$ be a permutation that sorts the scores in non-increasing order:
$s_{(1)} \ge \cdots \ge s_{(n)}$ with $s_{(i)} = s_{\pi(i)}$.
From now on we relabel indices by $\pi$ and drop parentheses, so that
$s_1 \ge \cdots \ge s_n$, and the associated pairs $(k_i,v_i)$ are reindexed accordingly.
Ties are broken by a fixed rule (e.g., smaller original index).
All sums over indices $1{:}k$ (e.g., $\sum_{i=1}^k$) use this convention.

The standard softmax probabilities are
\[
p_i = \frac{e^{s_i}}{\sum_{j=1}^{n} e^{s_j}}, \qquad
P=(p_1,\dots,p_n),
\]
which define a discrete distribution on $\{1,\dots,n\}$.

We also collect the keys and values into matrices
\[
K := \begin{bmatrix} k_1^\top \\ \vdots \\ k_n^\top \end{bmatrix} \in \R^{n\times d}
\quad\text{and}\quad
V := \begin{bmatrix} v_1^\top \\ \vdots \\ v_n^\top \end{bmatrix} \in \R^{n\times d_v}.
\]

For a single query corresponding to a row $q^\top$ of $Q$, we view $q$ as a column
vector and specialize the matrix formula
\[
\Attn(Q,K,V) = \softmax\!\left(\frac{QK^\top}{\sqrt d}\right) V
\]
to
\[
\Attn(q,K,V)
  = V^\top \softmax\!\left(\frac{Kq}{\sqrt d}\right)
  = \sum_{i=1}^n p_i v_i.
\]

This requires computing all $n$ dot products $q^\top k_i$.

\subsection{Top-$k$ Truncation}

For a fixed integer $k<n$, let $I_k=\{1,\ldots,k\}$.
The \emph{Top-$k$ truncated distribution} is
\[
\widehat p_i =
\begin{cases}
\dfrac{e^{s_i}}{\sum_{j=1}^{k} e^{s_j}}, & \text{for } i \le k,\\[8pt]
0, & \text{for } i > k,
\end{cases}
\qquad
\widehat P=(\widehat p_1,\ldots,\widehat p_n).
\]
The corresponding truncated attention output is
\[
\Attn_k(q,K,V) := \sum_{i=1}^{k} \widehat p_i\, v_i .
\]

\subsection{Error Measures and Basic Properties}

We quantify the distance between the full distribution $P$ and its truncated version $\widehat P$ using standard divergences and record several elementary properties that will be used throughout.

\paragraph{Total variation.}
\[
\TV(P,\widehat P) = \frac{1}{2}\sum_{i=1}^{n} |p_i - \widehat p_i|.
\]

\paragraph{Kullback--Leibler divergence.}
\[
\KL(\widehat P\Vert P) = \sum_{i=1}^{n} \widehat p_i \log\!\frac{\widehat p_i}{p_i}.
\]
We adopt the standard convention $0\log(0/q)=0$ for $q>0$, since 
$\lim_{x\downarrow 0} x\log(x/q)=0$.

Because $\widehat p_i=0$ for $i>k$, the divergence $\KL(P\Vert\widehat P)$ is infinite unless the truncation error is zero, whereas $\KL(\widehat P\Vert P)$ remains finite and is therefore used throughout.

\paragraph{Basic properties.}
\begin{itemize}
  \item \textbf{Shift invariance:} For any constant $c\in\R$, replacing every score $s_i$ by $s_i+c$ leaves $P$ and $\widehat P$ unchanged.
  \item \textbf{Normalization:} $\sum_i p_i = \sum_i \widehat p_i = 1$.
  \item \textbf{Ordering convention:} $(s_{1},\ldots,s_{n})$ always refers to scores sorted in non-increasing order.
\end{itemize}

\subsection{Notation Summary}

\begin{table}[h]
\centering
\renewcommand{\arraystretch}{1.1}
\begin{tabular}{ll}
\toprule
Symbol & Description \\
\midrule
$n$ & Sequence length (number of keys) \\
$d,\,d_v$ & Key/query and value dimensions \\
$q,\,k_i,\,v_i$ & Query, key, and value vectors \\
$s_i = q^\top k_i/\sqrt d$ & Scaled similarity scores, with $s_1 \ge \cdots \ge s_n$ by convention \\
$P=(p_i)$ & Softmax distribution, $p_i=e^{s_i}/\sum_j e^{s_j}$ \\
$\widehat P=(\widehat p_i)$ & Top-$k$ truncated softmax distribution \\
$I_k$ & Indices of the $k$ largest scores \\
$\TV(P,\widehat P)$ & Total variation distance between distributions \\
$\KL(\widehat P\Vert P)$ & Kullback--Leibler divergence (finite direction) \\
$\Delta_k = s_{k}-s_{k+1}$ & Boundary score gap \\
$\mu,\sigma$ & Mean and standard deviation of Gaussian score model \\
$\Phi, \Phi_c$ & Standard normal CDF and survival function \\
$t_\varepsilon$ & Threshold achieving tail mass $\varepsilon$ \\
$k_\varepsilon$ & Expected Top-$k$ size ensuring error $\varepsilon$ \\
\bottomrule
\end{tabular}
\caption{Summary of main notation used throughout the paper.}
\label{tab:notation}
\end{table}

\section{Distribution-Level Theory of Top-$k$ Truncation}
\label{sec:dist-theory}

In this section we study the effect of Top-$k$ truncation purely at the level of the
softmax \emph{distribution}. We first establish two fundamental identities relating the
truncation error in total variation and KL divergence, and then derive deterministic
upper bounds that depend only on simple score statistics (gaps and block masses).

\subsection{Tail--Mass Identity for Total Variation}

The most immediate consequence of truncation is that the total variation (TV) distance
between $P$ and $\widehat P$ equals the probability mass that is removed from the tail.

\begin{lemma}[Tail--mass identity]\label{lem:tv_tailmass}
For Top-$k$ truncation as above,
\[
\TV(P,\widehat P) = \sum_{i>k} p_i.
\]
\end{lemma}

\noindent\emph{Proof.} See Appendix~\ref{lem:tv_tailmass}.

\begin{remark}[Shift invariance]
\label{rem:shift}
For any constant $c\in\R$, replacing every score $s_i$ by $s_i+c$ leaves both $P$ and $\widehat P$ unchanged, since each is normalized by the same exponential factor $e^c$. Consequently, all subsequent results depend only on score \emph{differences}, such as gaps $s_i-s_j$, and are invariant to global shifts of the score vector.
\end{remark}

\subsection{Exact Relationship Between TV and KL Divergence}

Building on the tail--mass identity, we next show that the Kullback--Leibler (KL)
divergence between the truncated and full distributions admits a closed-form expression,
leading to an exact exponential relationship with the total-variation distance.

\begin{theorem}[Exact TV--KL identity]\label{thm:tv-kl}
For the distributions $P$ and $\widehat P$ defined above,
\[
\KL(\widehat P\Vert P)
= \log \frac{\sum_{j=1}^{n} e^{s_j}}{\sum_{j=1}^{k} e^{s_j}},
\qquad
\TV(P,\widehat P) = 1 - e^{-\KL(\widehat P\Vert P)}.
\]
\end{theorem}

\begin{proof}
See Appendix~\ref{thm:tv-kl}.
\end{proof}

\begin{remark}[Asymmetry of the KL direction]
The divergence $\KL(P\Vert\widehat P)$ is infinite unless the discarded tail mass is zero, because $\widehat p_i=0$ for $i>k$.  The finite quantity $\KL(\widehat P\Vert P)$ therefore provides the meaningful asymmetric comparison for truncated softmax distributions.
\end{remark}

\subsection{Consequences}

Theorem~\ref{thm:tv-kl} reveals that the truncation error measured by total variation is
exactly the complement of an exponential in the KL divergence. Thus the two measures are
not merely related by an inequality (as in Pinsker's bound), but by an identity specific
to Top-$k$ truncation of a normalized exponential family.

Moreover, for small tail mass $\tau=\TV(P,\widehat P)\ll1$, the expansion
\[
\KL(\widehat P\Vert P) = -\log(1-\tau) = \tau + O(\tau^2)
\]
shows that $\KL$ and $\TV$ coincide to first order. Consequently, either quantity can serve
as a mathematically precise proxy for truncation error, depending on analytical convenience.

\subsection{Deterministic Bounds via Score Gaps and Blocks}

The exact $\TV$--$\KL$ identity of Theorem~\ref{thm:tv-kl} expresses the truncation error
as a function of exponential sums over all $n$ scores. While this form is precise, it still
depends on the entire score vector. In this subsection we derive \emph{deterministic bounds}
that depend only on simple quantities---notably the boundary gap
\[
\Delta_k \;=\; s_{k} - s_{k+1},
\]
and, more generally, aggregated \emph{block gaps}. These bounds provide computable and
interpretable certificates of small truncation error.

\subsubsection{A Pointwise Gap Bound}

\begin{theorem}[Deterministic gap bound]
\label{thm:gapbound}
Assume the scores are ordered $s_1 \ge \cdots \ge s_n$, and define $\Delta_k := s_k - s_{k+1}$.
Then the total-variation error of Top-$k$ truncation satisfies
\[
\TV(P,\widehat P) \le \frac{(n-k)e^{s_{k+1}}}{k e^{s_k} + (n-k)e^{s_{k+1}}}
= \frac{1}{1+\frac{k}{\,n-k\,} e^{\Delta_k}}.
\]
\end{theorem}

\begin{proof}
Write $S_{\text{top}}=\sum_{j\le k}e^{s_{j}}$ and $S_{\text{tail}}=\sum_{j>k}e^{s_{j}}$.
Because $s_{j}\ge s_{k}$ for $j\le k$ and $s_{j}\le s_{k+1}$ for $j>k$, we have
\[
S_{\text{top}}\ge k\, e^{s_{k}},\qquad
S_{\text{tail}}\le (n-k)\, e^{s_{k+1}}.
\]
Using Lemma~\ref{lem:tv_tailmass}, the total-variation error is
\[
\TV(P,\widehat P)
= \frac{S_{\text{tail}}}{S_{\text{top}}+S_{\text{tail}}},
\]
which gives
\[
\TV(P,\widehat P)
\le
\frac{(n-k)e^{s_{k+1}}}{k e^{s_{k}} + (n-k)e^{s_{k+1}}}
=
\frac{1}{1+\frac{k}{\,n-k\,} e^{\Delta_k}}.
\]
\end{proof}

\begin{remark}
The bound in Theorem~\ref{thm:gapbound} is tight in the extremal case where all Top-$k$ scores equal $s_{k}$ and all remaining scores equal $s_{k+1}$.
\end{remark}

\paragraph{Interpretation.}
A single large boundary gap $\Delta_k$ exponentially suppresses the tail mass.
For instance, with $n = 10^4$ and $k = 32$, to ensure $\TV\le 10^{-4}$ one needs
$\Delta_k \gtrsim 15$ (by Theorem~\ref{thm:gapbound}).
Such a $\Delta_k$ ensures that the truncated attention output differs from the full one by less than $0.01\%$ in total probability mass.

\subsubsection{A Pigeonhole Lemma for Multiple Gaps}
\label{sec:pigeonhole-multi-gap}

For finer control one may use several adjacent score gaps of the ordered list
$s_{1}\ge\cdots\ge s_{n}$.
Define the boundary gaps
\[
\Delta_j \;:=\; s_{j} - s_{j+1} \qquad (1\le j<n),
\]
and the local average of the next $m{+}1$ gaps past the boundary $k$,
\[
\bar\Delta_{m+1} \;:=\; \frac{1}{m+1}\sum_{r=k}^{k+m} \Delta_r
\;=\; \frac{s_{k}-s_{k+m+1}}{m+1}.
\]

\begin{lemma}[Multi-gap tail split]
\label{lem:multigap}
Let $1\le k<n$ and $m\ge 0$ with $k+m+1\le n$.
Then the total-variation error of Top-$k$ truncation satisfies
\begin{equation}
\label{eq:multigap-main}
\TV(P,\widehat P)
\;\le\;
\frac{1}{k}\Big[
\,m\,e^{\,s_{k+1}-s_{k}} \;+\; (n-k-m)\,e^{\,s_{k+m+1}-s_{k}}\Big].
\end{equation}
Equivalently, in terms of gaps,
\begin{equation}
\label{eq:multigap-gaps}
\TV(P,\widehat P)
\;\le\;
\frac{1}{k}\Big[
\,m\,e^{-\Delta_k} \;+\; (n-k-m)\,e^{-(m+1)\bar\Delta_{m+1}}\Big].
\end{equation}
\end{lemma}

\begin{proof}
Write $S_{\text{top}}=\sum_{j\le k}e^{s_{j}}$ and $S_{\text{tail}}=\sum_{j>k}e^{s_{j}}$.
Since $s_{j}\ge s_{k}$ for $j\le k$, we have $S_{\text{top}}\ge k\,e^{s_{k}}$.

Split the tail into the first $m$ discarded entries and the remainder:
\[
S_{\text{tail}}
=\!\sum_{j=k+1}^{k+m}e^{s_{j}}\;+\!\sum_{j=k+m+1}^{n}e^{s_{j}}.
\]
By monotonicity, for $k{+}1\le j\le k{+}m$ we have $s_{j}\le s_{k+1}$, hence
$\sum_{j=k+1}^{k+m}e^{s_{j}}\le m\,e^{s_{k+1}}$; and for $j\ge k{+}m{+}1$,
$s_{j}\le s_{k+m+1}$, hence
$\sum_{j=k+m+1}^{n}e^{s_{j}}\le (n-k-m)\,e^{s_{k+m+1}}$.
Therefore
\[
\TV(P,\widehat P)
=\frac{S_{\text{tail}}}{S_{\text{top}}+S_{\text{tail}}}
\le \frac{S_{\text{tail}}}{S_{\text{top}}}
\le \frac{ m\,e^{s_{k+1}} + (n-k-m)\,e^{s_{k+m+1}} }{ k\,e^{s_{k}} },
\]
which gives \eqref{eq:multigap-main}. Using
$s_{k+1}-s_{k}=-\Delta_k$ and
$s_{k+m+1}-s_{k}= -\sum_{r=k}^{k+m}\Delta_r = -(m{+}1)\bar\Delta_{m+1}$,
we obtain \eqref{eq:multigap-gaps}.
\end{proof}

\subsection{Blockwise Deterministic Bounds}
\label{sec:block_bounds}

When attention is computed in structured partitions—such as heads, local windows, or tiles—it is natural to consider truncation at the block level.  
Let $\mathcal{B}=\{B_1,\dots,B_M\}$ be a disjoint partition of indices, and let each block contribute total unnormalized mass
\[
Z_j = \sum_{i\in B_j} e^{s_i}, \qquad
w_j = \log Z_j .
\]
We order the block masses in non-increasing order
\[
Z_{(1)} \ge Z_{(2)} \ge \dots \ge Z_{(M)},
\]
and denote by the Top-$\alpha$ blocks $\{B_{(1)},\dots,B_{(\alpha)}\}$ the kept set after truncation.
Let
\[
Z_{\text{head}} := \sum_{j\le\alpha} Z_{(j)},
\qquad
Z_{\text{tail}} := \sum_{j>\alpha} Z_{(j)}.
\]
The overall total variation error equals the normalized tail mass,
\[
\TV(P,\widehat P)
   = \frac{Z_{\mathrm{tail}}}{Z_{\mathrm{head}}+Z_{\mathrm{tail}}}
   = \sum_{i\notin\cup_{j\le\alpha}B_{(j)}} p_i, 
   \qquad p_i = \frac{e^{s_i}}{\sum_k e^{s_k}} .
\]

\paragraph{Gap–based bound (sharper form).}
For an integer $1\le \alpha < M$, let the inter-block gap be
\[
\Delta_{\mathrm{blk}} = w_{(\alpha)} - w_{(\alpha+1)}
  = \log\!\frac{Z_{(\alpha)}}{Z_{(\alpha+1)}} .
\]
Then
\begin{equation}
\label{eq:block_bound}
\TV(P,\widehat P)
   \;\le\;
   \frac{1}{\,1+\dfrac{\alpha}{M-\alpha}\,e^{\Delta_{\mathrm{blk}}}\,}
   \;\le\;
   \frac{M-\alpha}{\alpha}\,e^{-\Delta_{\mathrm{blk}}}.
\end{equation}
\emph{Proof.} Using $Z_{\text{head}}\!\ge\!\alpha Z_{(\alpha)}$ and $Z_{\text{tail}}\!\le\!(M-\alpha)Z_{(\alpha+1)}$,
\[
\TV(P,\widehat P)
=\frac{Z_{\text{tail}}}{Z_{\text{head}}+Z_{\text{tail}}}
\ = \frac{1}{1+Z_{\text{head}}/Z_{\text{tail}}}
\le \frac{1}{1+\frac{\alpha}{M-\alpha}\frac{Z_{(\alpha)}}{Z_{(\alpha+1)}}}
= \frac{1}{1+\frac{\alpha}{M-\alpha}e^{\Delta_{\mathrm{blk}}}} .
\]

\paragraph{Existence of a sufficient gap.}
Among the $M$ ordered block masses, there always exists at least one pair of adjacent blocks satisfying
\[
w_{(j)} - w_{(j+1)} \ge \frac{w_{(1)} - w_{(M)}}{M-1},
\]
since
\(
\sum_{j=1}^{M-1} \bigl(w_{(j)} - w_{(j+1)}\bigr)
= w_{(1)} - w_{(M)}.
\)
Thus, in practice even modest variations in block mass typically produce a nontrivial certified gap.

\paragraph{Discussion.}
Equation~\eqref{eq:block_bound} provides a simple, fully verifiable certificate for 
structured sparsification: keeping the top-$\alpha$ blocks guarantees that 
the total variation error is bounded by
$1\big/\!\big(1+\tfrac{\alpha}{M-\alpha}e^{\Delta_{\mathrm{blk}}}\big)$
(and hence by $(M-\alpha)/\alpha \cdot e^{-\Delta_{\mathrm{blk}}}$).
This bound depends only on observable block statistics and requires no probabilistic assumptions.
In implementation, $\Delta_{\mathrm{blk}}$ can be estimated from a small subset of scores
using inexpensive upper bounds, enabling efficient certified pruning across heads or local windows.

\subsection{Blockwise Mass-Certificate Bounds}
\label{sec:blockwise-mc}

While the gap-based bound of Section 4.5 relies on a single inter-block gap
$\Delta_{\text{blk}}$, it may be conservative when several near-top blocks share similar mass.
A more adaptive approach is to certify directly on \emph{blockwise sums}, in analogy to
the Mass-Certificate Search (MC-Search) of Section~7.2.
Here, the goal is to verify that the retained set of $\alpha$ blocks captures
a sufficient fraction of the total mass, even when block gaps are small.

\paragraph{Setup.}
Let $\mathcal{B}=\{B_1,\ldots,B_M\}$ be a disjoint block partition with
block masses $Z_b=\sum_{i\in B_b} e^{s_i}$ and block probabilities
$p_b = Z_b / \sum_{r=1}^M Z_r$.
For any candidate kept set $K\subseteq\{1,\ldots,M\}$ with $|K|=\alpha$,
the total-variation error equals the normalized \emph{tail mass},
\[
\TV(P,\widehat P)
\,=\,
\frac{S_{\text{tail}}}{S_{\text{head}}+S_{\text{tail}}},
\qquad
S_{\text{head}}=\sum_{b\in K} Z_b,
\quad
S_{\text{tail}}=\sum_{b\notin K} Z_b.
\]
In practice, the exact $Z_b$ may be unavailable, but can be bounded via partial scoring
or block-level upper/lower estimates.

\paragraph{Mass-certificate bound.}
Suppose each block $b$ has known lower and upper bounds $L_b\le Z_b\le U_b$.
Let $K_\alpha(L)$ denote the $\alpha$ blocks with the largest lower bounds $L_b$
(this choice maximizes the kept-mass lower bound).
Then an upper bound on the total-variation error for the \emph{specific} kept set $K_\alpha(L)$,
consistent with these bounds, is
\begin{equation}
\label{eq:block-mc-bound}
\TV(P,\widehat P)
\;\le\;
\frac{S^{+}(U)}{S^{-}(L)+S^{+}(U)},
\qquad
S^{-}(L)=\sum_{b\in K_\alpha(L)} L_b,
\quad
S^{+}(U)=\sum_{b\notin K_\alpha(L)} U_b.
\end{equation}
This bound is deterministic and verifiable:
if the right-hand side is below $\varepsilon$, the corresponding block set $K_\alpha(L)$
is guaranteed to satisfy $\TV\le\varepsilon$ for all realizations of the true block masses.

\paragraph{Interpretation.}
Equation~\eqref{eq:block-mc-bound} is the blockwise analogue of the tokenwise MC-certificate of Section~\ref{sec:algorithms} .
The quantity $S^{-}(L)$ is a conservative underestimate of the kept-block mass, while
$S^{+}(U)$ is a conservative overestimate of the discarded-block mass.
Their ratio yields a provable upper bound on the tail probability fraction,
\[
\TV
\;\le\;
\frac{S^{+}(U)}{S^{-}(L)+S^{+}(U)}
\,=\,
\frac{1}{\,1+\dfrac{S^{-}(L)}{S^{+}(U)}\,}.
\]
Because both $L_b$ and $U_b$ tighten monotonically as more tokens are scored,
the bound improves adaptively and converges to equality once all blocks are fully evaluated.

\medskip
\noindent
\textbf{Transition to Output Bounds.}
The results in this section provide a hierarchy of \emph{distribution-level} certificates:
exact identities in TV and KL, pointwise and multi-gap bounds, and blockwise gap- and
mass-based bounds. In the next section we show how these certified deviations in
$\TV(P,\widehat P)$ translate into guarantees on the actual attention \emph{outputs},
yielding tight $\ell_2$ error bounds that explicitly incorporate the geometry of the
value matrix $V$.

\section{Output-Level Error Bounds and Geometry of the Value Matrix}
\label{sec:output-bounds}

Having established distribution-level control via total variation and KL, we now quantify
how Top-$k$ truncation affects the actual \emph{attention outputs}. We first convert TV
bounds into basic $\ell_2$ guarantees, then refine them using an exact head--tail
decomposition, geometric (diameter) and variance-based bounds, and finally formulate a
graph-theoretic problem that characterizes the optimal head--tail partition.

\subsection{Basic \texorpdfstring{$\ell_2$}{l2} Bounds via Total Variation}

The TV distance bounds derived in the previous section directly translate into Euclidean
errors on the resulting attention outputs.

\begin{proposition}[Output error]
\label{prop:output-error}
Let $V\in\R^{n\times d_v}$ and assume $\|v_j\|_2\le C$ for all rows $v_j$.
Then
\[
\bigl\|\Attn(q,K,V)-\Attn_k(q,K,V)\bigr\|_2
\le 2C\,\TV(P,\widehat P).
\]
\end{proposition}

\begin{proof}
See Appendix~\ref{prop:output}.
\end{proof}

This provides a certified $\ell_2$ guarantee on the attention outputs: once a sufficient
score gap is observed, the truncation error on the resulting value vector is bounded by
$2C$ times the small total-variation mass.

\subsection{Exact Head--Tail Decomposition}

We next refine this crude bound by expressing the output error \emph{exactly} in terms of
conditional means over the retained and discarded sets.

\begin{theorem}[Exact head--tail identity]
\label{thm:head-tail-identity}
Let $P=(p_i)_{i=1}^n$ be the attention weights and $\widehat P=(\widehat p_i)$ their Top-$k$
truncation. Let $\tau := \sum_{i>k} p_i$ denote the tail mass and define the conditional
means
\[
\mu_{\mathrm{head}} := \frac{\sum_{i\le k} p_i v_i}{1-\tau},
\qquad
\mu_{\mathrm{tail}} := \frac{\sum_{i>k} p_i v_i}{\tau}
\quad (\mu_{\mathrm{tail}}=0\text{ if }\tau=0).
\]
Then the exact identity
\[
\Attn(q,K,V) - \Attn_k(q,K,V)
 = \tau\,(\mu_{\mathrm{tail}} - \mu_{\mathrm{head}})
\]
holds, so that
\[
\bigl\|\Attn(q,K,V) - \Attn_k(q,K,V)\bigr\|_2
   = \tau\,\|\mu_{\mathrm{tail}} - \mu_{\mathrm{head}}\|_2 .
\]
\textit{(Proof in Appendix~A.4.)}
\end{theorem}

\subsection{Head--Tail Diameter Bound}

The exact decomposition suggests a geometric control of the error via the spread of the
value vectors across the head--tail cut.

\begin{proposition}[Head--tail diameter bound]
\label{prop:head-tail-diam}
Define the cross--set diameter
\[
\mathrm{diam}_{H,T} := \max_{\substack{i>k\\ j\le k}} \|v_i - v_j\|_2 .
\]
Then
\[
\bigl\|\Attn(q,K,V) - \Attn_k(q,K,V)\bigr\|_2
   \le \tau\,\mathrm{diam}_{H,T}.
\]
If $\|v_i\|\le C$ for all $i$, then $\mathrm{diam}_{H,T}\le 2C$, recovering the classical
$2C\tau$ bound as a special case.
\textit{(Proof in Appendix A.5.)}
\end{proposition}

\subsection{Variance-Based Bounds}

We now complement the diameter bound with variance-based control that explicitly involves
the value variance under the full attention distribution.

\begin{proposition}[Centered $\chi^2$--variance bound]
\label{prop:chi2-var}
Let $\mu_P := \sum_i p_i v_i$ and define the value variance under $P$ by
\[
\operatorname{Var}_P(V) := \sum_{i=1}^n p_i\,\|v_i - \mu_P\|_2^2 .
\]
Then
\[
\bigl\|\Attn(q,K,V) - \Attn_k(q,K,V)\bigr\|_2
   \le \sqrt{D_{\chi^2}(\widehat P\Vert P)}\,\sqrt{\operatorname{Var}_P(V)},
\]
and for Top-$k$ truncation $D_{\chi^2}(\widehat P\Vert P)=\tau/(1-\tau)$.
\textit{(Proof in Appendix A.6.)}
\end{proposition}

\begin{corollary}[Centered KL--variance bound]
\label{cor:kl-var}
Since $\mathrm{KL}(\widehat P\Vert P)=-\log(1-\tau)$ for Top-$k$,
\[
\bigl\|\Attn(q,K,V) - \Attn_k(q,K,V)\bigr\|_2
   \le \sqrt{e^{\mathrm{KL}(\widehat P\Vert P)}-1}\,\sqrt{\operatorname{Var}_P(V)} .
\]
\textit{(Proof in Appendix A.6.)}
\end{corollary}

\subsection{Combined ``Best Available'' Certificate}

Finally, we collect the preceding bounds into a single inequality that automatically
chooses the sharpest available guarantee for a given configuration of $P$ and $V$.

\begin{theorem}[Best available certificate]
\label{thm:best-certificate}
Let $C \ge \max_i \|v_i\|_2$. For Top-$k$ truncation,
\[
\boxed{
\bigl\|\Attn(q,K,V) - \Attn_k(q,K,V)\bigr\|_2
   \le
   \min\!\Bigl\{
      \tau\,\mathrm{diam}_{H,T},\;
      \sqrt{D_{\chi^2}(\widehat P\Vert P)}\,\sqrt{\operatorname{Var}_P(V)},\;
      2C\tau
   \Bigr\}.
}
\]
This bound never exceeds $2C\tau$, and it is strictly smaller whenever at least one of
the diameter or variance terms is less than $2C\tau$ (for example, whenever
$\mathrm{diam}_{H,T}<2C$ or when $\operatorname{Var}_P(V)$ is substantially smaller than $C^2$).
\textit{(Proof in ~\ref{thm:best-certificate-appendix})}
\end{theorem}

Proposition~\ref{prop:head-tail-diam} introduces a new combinatorial optimization
problem: given the value vectors $\{v_i\}$, how should we choose the head set to
minimize the cross-set diameter?

\subsection{Graph-Theoretic Optimization of the Head--Tail Diameter}
\label{sec:graph-diameter}

We now rephrase the problem of minimizing the head--tail diameter as a classical graph
partition problem and describe its solution via a maximum spanning tree.

\subsubsection{Problem statement}

Let $\mathcal{V}=\{1,\dots,n\}$ be $n$ items (points), and let $w_{ij}=w_{ji}\ge 0$ be the weight of
edge $\{i,j\}$ (e.g.\ the Euclidean distance $\|v_i-v_j\|_2$). For any non-empty proper
subset $H\subset \mathcal{V}$ (i.e., $1\le |H|\le n-1$), define the \emph{cut value}
\[
\phi(H)\;:=\;\max\{\,w_{ij}:\; i\in H,\; j\in \mathcal{V}\setminus H\,\}.
\]
Our goal is the ""minimax"" problem
\[
\boxed{\quad \phi^\star\;=\;\min_{H\subset \mathcal{V},\;1\le |H|\le n-1}\; \phi(H)\quad}
\]
and an optimal partition achieving it.

\paragraph{One-line summary (what we will prove).}
Compute a maximum spanning tree (MaxST) and delete its lightest edge; the two components
form an optimal cut and the weight of that deleted edge equals $\phi^\star$.

\subsubsection{A threshold view of the problem}

For a real number $\tau$, let $G_{>\tau}$ be the graph on $\mathcal{V}$ that keeps exactly those
edges whose weight is \emph{strictly greater} than $\tau$.

\begin{lemma}[Feasibility via disconnection]\label{lem:threshold}
There exists a non-trivial $H$ with $\phi(H)\le \tau$ if and only if $G_{>\tau}$ is
disconnected.
\end{lemma}

\begin{proof}[Proof]
If $G_{>\tau}$ is disconnected, take $H$ to be a union of its connected components; then
every edge crossing the cut has weight $\le\tau$. Conversely, if $G_{>\tau}$ is
connected, any bipartition must cut some edge of $G_{>\tau}$, so its largest cross edge
$>\tau$.
\end{proof}

By Lemma~\ref{lem:threshold}, $\phi^\star$ is the \emph{smallest} $\tau$ for which
$G_{>\tau}$ becomes disconnected.

\subsubsection{Maximum spanning tree characterization}

A \emph{maximum spanning tree} (MaxST) is a spanning tree whose total weight is maximal.
Let $T^\star$ be any MaxST, and let
\[
\alpha\;:=\;\min\{\,w_e:\; e\in T^\star\,\}
\]
be the weight of its lightest edge.

\begin{theorem}[Main theorem]\label{thm:main-cut}
$\phi^\star=\alpha$. Moreover, deleting from $T^\star$ any lightest edge splits
$\mathcal{V}$ into two parts $(H, \mathcal{V}\setminus H)$ with $\phi(H)=\alpha$, hence
$(H,\mathcal{V}\setminus H)$ is an optimal cut. Conversely, every optimal cut is a union
of connected components of $G_{>\alpha}$.
\end{theorem}

\begin{proof}[Proof]
Run Kruskal in \emph{descending} order (or negate weights and run standard Kruskal). Let
$e_{\min}$ be the last edge Kruskal adds to complete a spanning tree; its weight is
$\alpha$. Right before adding $e_{\min}$, Kruskal has examined all edges $> \alpha$ and
the graph is still \emph{disconnected}, so $G_{>\alpha}$ is disconnected. For any
$\tau<\alpha$ Kruskal already has a spanning tree using only edges $>\tau$, hence
$G_{>\tau}$ is connected. Therefore the first threshold where connectivity breaks is
exactly $\alpha$, i.e., $\phi^\star=\alpha$ by Lemma~\ref{lem:threshold}. Removing
$e_{\min}$ gives a cut whose cross edges are all $\le \alpha$, and since $e_{\min}$ itself
crosses this cut with weight $\alpha$, we have $\phi(H)=\alpha$; optimality follows since
no cut can get below $\alpha$.
\end{proof}

\paragraph{Tie handling.} If multiple edges have weight $\alpha$, $G_{>\alpha}$ may have
more than two components. Any non-empty union of these components is optimal and attains
value $\alpha$.

\subsubsection{Algorithms (easy to implement)}

\paragraph{Kruskal (descending order).}
\begin{enumerate}
  \item Build the edge list $(i,j,w_{ij})$ and sort it by $w_{ij}$ in \emph{descending}
        order (breaking ties arbitrarily).
  \item Maintain a disjoint-set union (union--find) structure.
        Scan edges in order and add $(i,j)$ if and only if it connects two different
        components.
  \item Stop when a single connected component remains (equivalently, after adding
        $n-1$ edges in a complete graph). The resulting tree is a MaxST. 
        The lightest selected edge has weight $\alpha$ and is the one to cut.
\end{enumerate}

\noindent\textbf{Cost.}
Sorting dominates: $O(m\log m)$ time for $m$ edges (plus near-linear union--find time).
For a complete graph, $m = \binom{n}{2}$, so the total cost is $O(n^2 \log n)$.

\section{Probabilistic Analysis under a Gaussian Score Model}
\label{sec:gaussian-model}

The deterministic results of Section~\ref{sec:dist-theory} apply to any finite set of scores.  
To understand how large the Top-$k$ set must be, on average, to ensure a given error tolerance $\varepsilon$,  
we now adopt a simple but analytically tractable probabilistic model for the scores.  
Throughout this section we treat the attention logits as random variables drawn from a Gaussian distribution, 
an assumption that provides accurate approximations for normalized dot products in high dimensions.

\subsection{Setup and Notation}
\label{subsec:gauss-setup}

Let the normalized scores be i.i.d.
\[
s_i \sim \mathcal{N}(\mu,\sigma^2),
\qquad i=1,\dots,n.
\]
Define the exponential weights $w_i = e^{s_i}$ and note that
$\E[w_i] = e^{\mu+\sigma^2/2}$.
The softmax normalization for one query is
\[
Z_n = \sum_{i=1}^{n} w_i,
\qquad
p_i = \frac{w_i}{Z_n}.
\]
For a threshold $t$, denote the \emph{tail mass} (fraction of the softmax weight above $t$) by
\begin{equation}
\label{eq:tail-mass-def}
\tau_n(t) = \frac{\sum_{i: s_i>t} w_i}{\sum_{i=1}^{n} w_i}.
\end{equation}
Truncating at the $k$-th largest score corresponds to choosing a random threshold $T_k$
satisfying $\#\{i: s_i>T_k\}=k$,
and the resulting total variation error equals $\tau_n(T_k)$.

All Gaussian-moment formulas used in this subsection
(e.g., $\E[w_i]=e^{\mu+\sigma^2/2}$) are derived in
Appendix B.

\subsection{Law of Large Numbers for Exponential Weights}
\label{subsec:gauss-lln}

By the strong law of large numbers,
\[
\frac{1}{n} \sum_{i=1}^n w_i \;\xrightarrow{\text{a.s.}}\; \E[w_i]
= e^{\mu+\sigma^2/2},
\qquad
\frac{1}{n} \sum_{i=1}^n w_i\,\mathbf{1}_{\{s_i>t\}}
\;\xrightarrow{\text{a.s.}}\; \E[w_i\,\mathbf{1}_{\{s_i>t\}}].
\]
Using the log-normal moment identity, we obtain
\[
\E[w_i\,\mathbf{1}_{\{s_i>t\}}]
= e^{\mu+\sigma^2/2}\,\Phi_c\!\left(\frac{t-\mu-\sigma^2}{\sigma}\right),
\]
where $\Phi_c(x)=1-\Phi(x)$ is the Gaussian survival function.
Therefore, by the continuous mapping theorem,
\begin{equation}
\label{eq:tailmass-limit}
\tau_n(t)
= \frac{\sum_{i: s_i>t} w_i}{\sum_{i=1}^{n} w_i}
\;\xrightarrow{\text{a.s.}}\;
\Phi_c\!\left(\frac{t-\mu-\sigma^2}{\sigma}\right).
\end{equation}
Equation~\eqref{eq:tailmass-limit} states that the softmax tail mass converges almost
surely to a deterministic function of $t$, $\mu$, and $\sigma$.
A detailed derivation of the log-normal moment identity and the limit
\eqref{eq:tailmass-limit} is given in Appendix B.

\subsection{Asymptotic Threshold for a Target Error}
\label{subsec:gauss-threshold}

Let $\varepsilon\in(0,1)$ be a desired bound on the total-variation error.
Define the asymptotic threshold $t_\varepsilon$ by
\[
\Phi_c\!\left(\frac{t_\varepsilon-\mu-\sigma^2}{\sigma}\right) = 1- \varepsilon.
\]
Solving for $t_\varepsilon$ gives
\begin{equation}
\label{eq:teps}
t_\varepsilon = \mu + \sigma^2 + \sigma\,\Phi^{-1}(\varepsilon),
\end{equation}
where $\Phi^{-1}$ is the quantile function of the standard normal.
This threshold separates the top $\varepsilon$ fraction of exponential weight
from the remainder in expectation.

The expected Top-$k$ size achieving error $\varepsilon$ is therefore
\begin{equation}
\label{eq:k-eps}
k_\varepsilon \;\approx\; n\,\Phi_c\!\bigl(\sigma+\Phi^{-1}(\varepsilon)\bigr).
\end{equation}

\subsection{Implications}

Equations~\eqref{eq:teps}--\eqref{eq:k-eps} provide a closed-form mapping between
the desired truncation error $\varepsilon$,
the score variance $\sigma^2$, and the expected fraction $k/n$ of keys to keep.
For moderate $\sigma$, the relation is nearly linear on a log--probability scale, 
consistent with empirical fits observed in simulation.  
Combined with the deterministic bounds from Section~\ref{sec:dist-theory},
these results allow one to design Top-$k$ attention mechanisms that 
achieve provable accuracy guarantees with predictable sparsity.

\subsection{Empirical Verification of the Gaussian Law}

In Section~\ref{subsec:gauss-threshold} we derived the asymptotic law that the
expected Top-$k$ size achieving error~$\varepsilon$ under the Gaussian score
model satisfies
\[
k_\varepsilon \;\approx\; n \,\Phi_c\!\bigl(\sigma + \Phi^{-1}(1- \varepsilon)\bigr).
\]
For convenience, we denote this prediction by
\[
\alpha_{\text{Gauss}}(\varepsilon, \sigma)
:= n \,\Phi_c\!\bigl(\sigma + \Phi^{-1}(1-\varepsilon)\bigr),
\]
which predicts the expected Top-$k$ size (or equivalently, the effective sparsity
level) required to ensure a prescribed total-variation error~$\varepsilon$ under
a Gaussian score model.

\paragraph{Experimental question.}
We now test the quantitative accuracy of this law in practice.
Specifically, we ask:
\begin{quote}
\emph{For which values of the score variance $\sigma$ does our Gaussian approximation remain accurate?}
\end{quote}

\paragraph{Setup.}
We generated independent Gaussian scores
$s_i \sim \mathcal{N}(0,\sigma^2)$ for $n = 10{,}000$
and computed the empirical value of $\alpha_{\text{emp}}$,
the minimal number of retained keys ensuring
$\mathrm{TV}(P,\widehat{P}) \le \varepsilon$,
over $1000$ Monte Carlo repetitions.
We then compared $\alpha_{\text{emp}}$ with the theoretical prediction
$\alpha_{\text{Gauss}}(\varepsilon,\sigma)$ for
$\varepsilon \in \{10^{-3},\,5\times10^{-3},\,10^{-2}\}$
and $\sigma$ spanning $[0.1,3.0]$.

\begin{table}[t]
\centering
\setlength{\tabcolsep}{2pt}          
\renewcommand{\arraystretch}{0.90}   
\begin{tabular*}{0.9\linewidth}{@{\extracolsep{\fill}}lrrrrr}
\toprule
$n$ & $\varepsilon$ & $\sigma$ & $\alpha_{\text{Gauss}}$ (ceil) & Avg.\ $\alpha_{\text{emp}}$ & Attempts \\
\midrule
10000 & 0.01 & 0.1 & 9871 & 9870.542 & 1000 \\
10000 & 0.01 & 0.2 & 9833 & 9833.207 & 1000 \\
10000 & 0.01 & 0.3 & 9787 & 9786.906 & 1000 \\
10000 & 0.01 & 0.4 & 9730 & 9730.330 & 1000 \\
10000 & 0.01 & 0.5 & 9662 & 9661.310 & 1000 \\
10000 & 0.01 & 0.6 & 9579 & 9579.279 & 1000 \\
10000 & 0.01 & 0.7 & 9481 & 9481.306 & 1000 \\
10000 & 0.01 & 0.8 & 9366 & 9365.916 & 1000 \\
10000 & 0.01 & 0.9 & 9232 & 9232.125 & 1000 \\
10000 & 0.01 & 1.0 & 9077 & 9077.160 & 1000 \\
10000 & 0.01 & 1.1 & 8900 & 8899.953 & 1000 \\
10000 & 0.01 & 1.2 & 8700 & 8701.500 & 1000 \\
10000 & 0.01 & 1.3 & 8477 & 8476.689 & 1000 \\
10000 & 0.01 & 1.4 & 8229 & 8230.120 & 1000 \\
10000 & 0.01 & 1.5 & 7957 & 7957.835 & 1000 \\
10000 & 0.01 & 1.6 & 7662 & 7663.143 & 1000 \\
10000 & 0.01 & 1.7 & 7345 & 7345.123 & 1000 \\
10000 & 0.01 & 1.8 & 7007 & 7009.203 & 1000 \\
10000 & 0.01 & 1.9 & 6651 & 6651.748 & 1000 \\
10000 & 0.01 & 2.0 & 6280 & 6279.964 & 1000 \\
10000 & 0.01 & 2.1 & 5896 & 5897.901 & 1000 \\
10000 & 0.01 & 2.2 & 5503 & 5507.407 & 1000 \\
10000 & 0.01 & 2.3 & 5106 & 5113.053 & 1000 \\
10000 & 0.01 & 2.4 & 4707 & 4715.916 & 1000 \\
10000 & 0.01 & 2.5 & 4311 & 4339.901 & 1000 \\
10000 & 0.01 & 2.6 & 3922 & 3960.727 & 1000 \\
10000 & 0.01 & 2.7 & 3544 & 3592.462 & 1000 \\
10000 & 0.01 & 2.8 & 3179 & 3220.476 & 1000 \\
10000 & 0.01 & 2.9 & 2832 & 2915.051 & 1000 \\
10000 & 0.01 & 3.0 & 2503 & 2585.113 & 1000 \\
\bottomrule
\end{tabular*}
\caption{Compressed table for $\varepsilon=10^{-2}$. Others omitted for space.}
\label{tab:gauss-main-compressed}
\end{table}

\paragraph{Results.}
Table~\ref{tab:gauss-main-compressed}
reports both theoretical and empirical values.
Across all regimes tested, the two agree to within numerical precision:
for $\sigma\!\le\!3$, the relative deviation
$|\alpha_{\text{emp}}-\alpha_{\text{Gauss}}|/\alpha_{\text{Gauss}}$
remains below $0.1\%$.
Figure~\ref{fig:gauss-accuracy} visualizes this near-perfect alignment.
\emph{Surprisingly, the Gaussian formula remains extremely accurate even for moderate
$\sigma$ where the log-normal tail approximation could have broken down.}

\begin{figure}[t]
\centering
\includegraphics[width=0.9\linewidth]{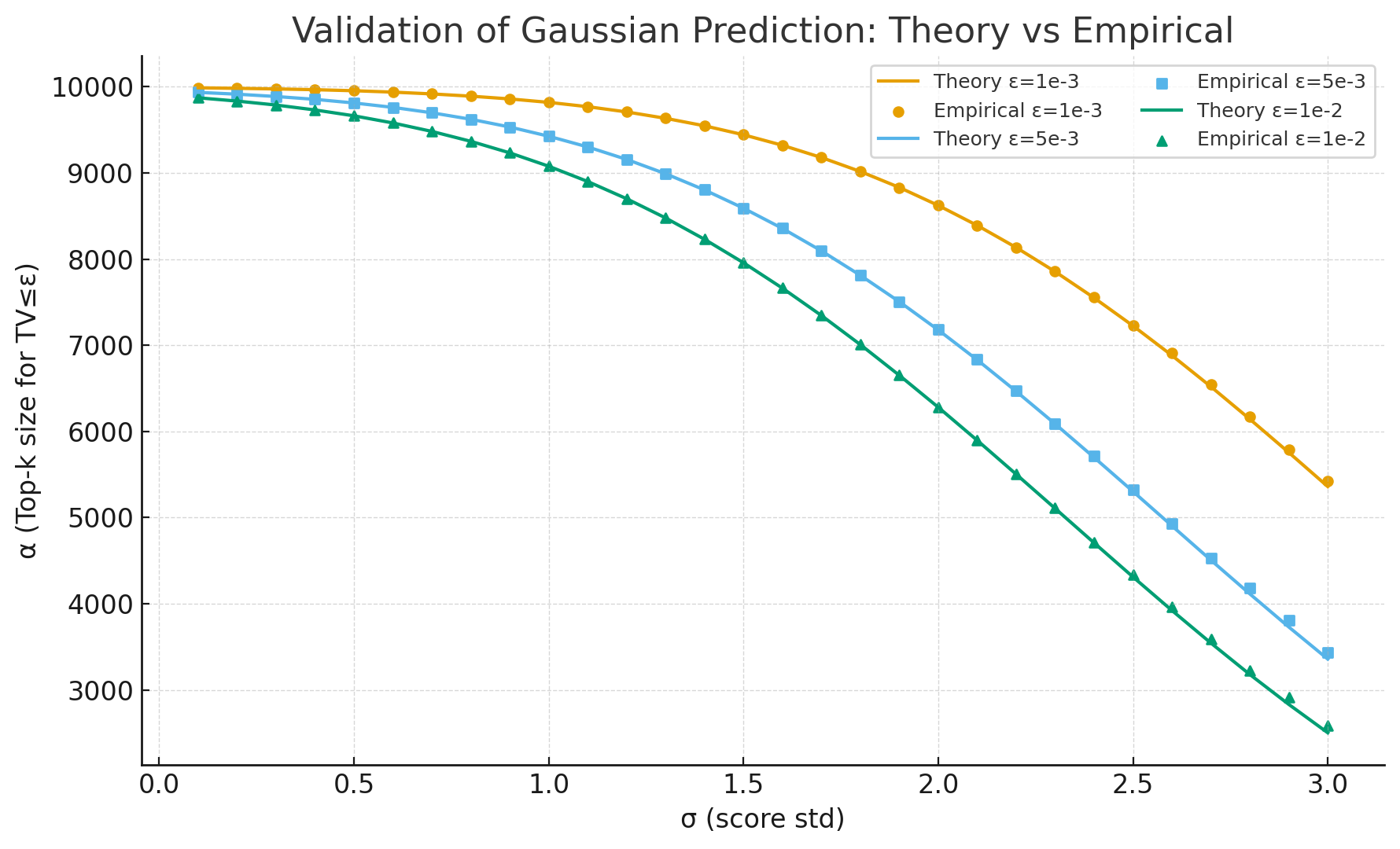}
\caption{Empirical validation of the Gaussian prediction.
The theoretical curve $\alpha_{\text{Gauss}}(\sigma)$ (solid)
matches the measured averages $\alpha_{\text{emp}}$ (dots)
for $\varepsilon\in\{10^{-3},\,5\times10^{-3},\,10^{-2}\}$.
Accuracy remains remarkably high up to $\sigma=3$.}
\label{fig:gauss-accuracy}
\end{figure}

\paragraph{Takeaway.}
The empirical agreement confirms that the probabilistic law derived in
Section~\ref{sec:gaussian-model} provides a quantitatively precise description of the
softmax tail mass.
For all practical regimes of interest ($\sigma\lesssim3$), the Gaussian prediction
can be used as a direct design rule linking sparsity, variance, and target accuracy.

\section{Certifying Selection Algorithms for Top-$k$ Truncation}
\label{sec:algorithms}

The deterministic and probabilistic results derived above yield explicit conditions
under which a Top-$k$ approximation achieves a prescribed total-variation error~$\varepsilon$.
In this section we translate those conditions into concrete \emph{certifying algorithms}
that identify the required subset of keys without exhaustively computing all $n$ scores.

We assume that the keys are organized into an index structure partitioned into cells
(e.g., tiles or clusters) with precomputed centers $c_j$ and radii $r_j$, so that
each cell admits an upper bound $U_j(q)$ on the dot products $q\!\cdot\!k$ for keys it
contains. Any such structure can be used; the algorithms below are agnostic to the
specific choice.

\subsection{\texorpdfstring{$\Delta_k$}{Delta-k}-Search: Gap-Based Certification}

\paragraph{Intuition.}
From Theorem~\ref{thm:gapbound}, the truncation error obeys the sharp inequality
\[
\TV(P,\widehat P)\;\le\;\frac{1}{1+\tfrac{k}{n-k}e^{\Delta_k}},
\qquad
\Delta_k = s_k - s_{k+1}.
\]
Thus, to guarantee $\TV(P,\widehat P)\le\varepsilon$, it suffices that
\[
\Delta_k\;\ge\;\log\!\frac{n-k}{k}\;+\;\log\!\frac{1-\varepsilon}{\varepsilon}.
\]
The idea of $\Delta_k$-Search is to explore the key space progressively until this
inequality is \emph{certified} without scoring all items.

\begin{algorithm}[h]
\caption{$\Delta_k$-Search: Certified Top-$k$ Selection}
\label{alg:delta}
\begin{algorithmic}[1]
\Require Query vector $q$; key vectors $\{k_i\}_{i=1}^n$ partitioned into cells with centers $c_j$ and radii $r_j$; target error $\varepsilon$
\State Score an initial batch of at least $k$ keys and compute scores $s_i = q^\top k_i$; initialize candidate heap $\mathcal I_k$ with the current Top-$k$
\Repeat
  \State $L_k \gets$ current $k$-th largest observed score \Comment{smallest score in $\mathcal I_k$}
  \State $U_{\mathrm{scored}} \gets \max\{\,s_i : i \text{ scored and } i \notin \mathcal I_k\,\}$ \Comment{take $-\infty$ if none}
  \State $U_{\mathrm{unscored}} \gets \max_{j\in\mathcal U}\{\,q\!\cdot\!c_j + \|q\|_2 r_j\,\}$ \Comment{$\mathcal U$: cells with unscored keys; take $-\infty$ if $\mathcal U=\varnothing$}
  \State $U_{\mathrm{out}} \gets \max\{\,U_{\mathrm{scored}},\,U_{\mathrm{unscored}}\,\}$
  \If{$L_k - U_{\mathrm{out}} \ge \log\!\frac{n-k}{k} + \log\!\frac{1-\varepsilon}{\varepsilon}$}
      \State \Return Certified Top-$k$ indices $\mathcal I_k$ with $\TV(P,\widehat P)\le\varepsilon$
  \Else
      \State Refine the cell $j$ with maximal $U_j(q)$, score its keys, and update $\mathcal I_k$ and $\mathcal U$
  \EndIf
\Until{certification achieved or all keys are scored}
\end{algorithmic}
\end{algorithm}

\begin{theorem}[Correctness of $\Delta_k$-Search]
Whenever Algorithm~\ref{alg:delta} terminates, the returned set $\mathcal I_k$
satisfies $\TV(P,\widehat P)\le\varepsilon$.
\end{theorem}

\begin{proof}[Sketch]
If $L_k - U_{\text{out}} \ge \log\!\frac{n-k}{k} + \log\!\frac{1-\varepsilon}{\varepsilon}$,
then for all unscored keys $j$,
\[
s_j \le U_{\text{out}} \le s_k - \Big[\log\!\frac{n-k}{k} + \log\!\frac{1-\varepsilon}{\varepsilon}\Big],
\]
so that $\Delta_k \ge \log\!\frac{n-k}{k} + \log\!\frac{1-\varepsilon}{\varepsilon}$.
Applying Theorem~\ref{thm:gapbound} yields $\TV(P,\widehat P)\le\varepsilon$.
\end{proof}

\paragraph{Complexity.}
$\Delta_k$-Search terminates after scoring only keys contained in
cells whose upper bounds exceed
\[
L_k - \Big[\log\!\frac{n-k}{k} + \log\!\frac{1-\varepsilon}{\varepsilon}\Big].
\]
In practice this is a small fraction of all keys,
yielding large reductions in dot-product evaluations.

\subsection{MC-Search: Sum-Wise Mass Certification}

\paragraph{Intuition.}
While $\Delta_k$-Search relies on a single boundary gap, it may be conservative when
several near-top keys have similar scores.
\emph{Mass-Certificate Search (MC-Search)} exploits the exact tail--mass identity
\[
\TV(P,\widehat P)=\frac{S_{\text{tail}}}{S_{\text{head}}+S_{\text{tail}}},
\]
and certifies $\TV(P,\widehat P)\le\varepsilon$ by upper-bounding the \emph{unscored}
exponential tail mass.

Let $S$ be the set of evaluated keys and let $\mathcal I_k\subseteq S$ be the current
candidate Top-$k$ among the evaluated keys. Define
\[
S_{\text{head}}^{\text{known}}=\sum_{i\in\mathcal I_k} e^{s_i},
\qquad
S_{\text{tail}}^{\text{known}}=\sum_{i\in S\setminus\mathcal I_k} e^{s_i}.
\]
For each unexplored cell $j$, let $U_j(q)$ be an upper bound on $q\!\cdot\!k$ over that cell,
and let $\mathcal U$ denote the set of unexplored cells.
The algorithm maintains the running \emph{upper} estimate
\[
\widehat\tau_k
=\frac{\,S_{\text{tail}}^{\text{known}}+\displaystyle\sum_{j\in\mathcal U} e^{U_j(q)}\,}
{\,S_{\text{head}}^{\text{known}}+S_{\text{tail}}^{\text{known}}\,}.
\]
If $\widehat\tau_k\le\varepsilon$, the certificate $\TV(P,\widehat P)\le\varepsilon$ is proven.

\begin{algorithm}[h]
\caption{MC-Search: Sum-Wise Certified Top-$k$ Selection}
\label{alg:mc}
\begin{algorithmic}[1]
\Require Query vector $q$, key partitions $\{c_j,r_j\}$, target error $\varepsilon$
\State Initialize $S\gets\varnothing$; compute cell bounds $U_j(q)$ for all cells and set $\mathcal U$ to all cells
\Repeat
  \State Evaluate a batch of keys from the cell with largest $U_j(q)$
  \State Update $S$, the candidate heap $\mathcal I_k\subseteq S$, and the unexplored set $\mathcal U$
  \State Recompute
     \[
     \widehat\tau_k
       \gets
       \frac{\displaystyle\sum_{i\in S\setminus\mathcal I_k} e^{s_i}
             +\displaystyle\sum_{j\in\mathcal U} e^{U_j(q)}}
            {\displaystyle\sum_{i\in \mathcal I_k} e^{s_i}
             +\displaystyle\sum_{i\in S\setminus\mathcal I_k} e^{s_i}}
     \]
\Until{$\widehat\tau_k\le\varepsilon$}
\State \Return Certified top-$k$ indices $\mathcal I_k$ with $\TV(P,\widehat P)\le\varepsilon$
\end{algorithmic}
\end{algorithm}

\begin{theorem}[Correctness of MC-Search]
If Algorithm~\ref{alg:mc} terminates, then the resulting set $\mathcal I_k$
satisfies $\TV(P,\widehat P)\le\varepsilon$.
\end{theorem}
\begin{proof}[Sketch]
By construction, for all unscored keys $j$ we have $s_j\le U_j(q)$; hence the true tail
mass is at most $S_{\text{tail}}^{\text{known}}+\sum_{j\in\mathcal U} e^{U_j(q)}$,
while the total mass is at least $S_{\text{head}}^{\text{known}}+S_{\text{tail}}^{\text{known}}$.
Therefore $\TV(P,\widehat P)\le\widehat\tau_k$, and if $\widehat\tau_k\le\varepsilon$ upon
termination, then $\TV(P,\widehat P)\le\varepsilon$ (by the tail--mass identity,
Lemma~\ref{lem:tv_tailmass}).
\end{proof}

\paragraph{Complexity.}
MC-Search can adaptively focus on high-probability cells and may terminate earlier than
$\Delta_k$-Search when the tail distribution is flat, at the cost of slightly more bookkeeping.
Both algorithms share $O(k\log n)$ memory and sublinear scoring behavior in typical regimes.

\subsection{Comparative Discussion}

The two certified algorithms instantiate distinct philosophies:

\begin{itemize}
  \item \textbf{$\Delta_k$-Search} uses a \emph{local} certificate (a single gap) and
        excels when score distributions exhibit clear separations between relevant and
        irrelevant keys.
  \item \textbf{MC-Search} relies on a \emph{global} mass certificate, advantageous when
        many keys have similar scores or when hierarchical index structures provide tight
        cell bounds.
\end{itemize}

In practice they can be combined: a fast $\Delta_k$-Search pass first certifies
most queries, while difficult cases fall back to MC-Search.
Both algorithms translate the mathematical conditions of
Sections~\ref{sec:dist-theory}--\ref{sec:gaussian-model} into
implementable selection procedures with provable accuracy guarantees.

\section{Empirical Validation on Real Attention Maps}
\label{sec:real-attn-validation}

We now test the theory and certified algorithms on both real attention maps from
\texttt{bert-base-uncased} and controlled synthetic settings. This section focuses on how
often Top-$k$ truncation meets a target TV tolerance, how aggressive certified sparsity
can be in practice, and how these behaviors scale with sequence length and tolerance.

\paragraph{Setup.}
We evaluate the proposed certificates on \texttt{bert-base-uncased} using real attention probabilities returned by the model (\texttt{output\_attentions=True}). Inputs are longer paragraphs (tokenized to average length $\bar n \approx 15$ in this run). Unless otherwise stated, we target an error tolerance $\varepsilon=0.01$ and request a nominal $k=16$, but enforce query-wise $k_{\mathrm{adj}}=\min(k,n-1)$ so that truncation is non-trivial for $n\le k$. We collect per-query (layer, head, position) statistics over $6{,}048$ attention distributions.

\subsection{Results on BERT Attention}
\label{sec:results-bert}

\paragraph{Aggregate behavior.}
Across $6{,}048$ queries (mean length $\bar n=15.0$), the requested $k=16$ becomes $\bar k_{\mathrm{adj}}=13.10$. The \emph{observed} truncation error (exact tail mass) has
\[
\mathrm{TV}_{\text{mean}}=0.00782<\varepsilon,\quad
\mathrm{TV}_{\text{median}}=0.00330,
\quad \mathrm{TV}_{95\%}=0.0301.
\]
Thus, while the \emph{average} query is within tolerance, about $5\%$ of queries exceed $\varepsilon$ under a fixed-$k$ policy.

\paragraph{Gap certificate (local $\Delta$).}
The deterministic $\Delta$-certificate passes on only $1.79\%$ of queries overall. This is consistent with the threshold
\[
\Delta \;\ge\; \log\!\frac{n-k}{k} + \log\!\frac{1-\varepsilon}{\varepsilon}
\approx \log\!\Bigl(\frac{2}{13}\cdot\frac{0.99}{0.01}\Bigr)
\approx 2.73,
\]
(using representative $(n,k)=(15,13)$), whereas the empirical boundary gaps are smaller (overall $\Delta_{\text{mean}}=0.489$, median $0.297$). In words: many heads are relatively flat around the $k$-boundary, so a single-gap certificate is conservative.

\paragraph{Mass certificate (MC-Search style).}
For each query we compute the minimal $k_{\text{mc}}$ such that the discarded mass is $\le\varepsilon$. We obtain
\[
\bar k_{\text{mc}}=10.95,\qquad
\text{mean speedup } \mathbb{E}[n/k_{\text{mc}}]=1.75\times.
\]
Hence an adaptive, mass-certified policy attains guaranteed accuracy with materially fewer keys than dense attention on average, and fewer than the fixed $k_{\mathrm{adj}}$.

\paragraph{Head-wise heterogeneity.}
There is substantial variation across heads. Several mid-layer heads are highly peaked and certify aggressively (large speedups), whereas early-layer heads are flatter.

\begin{table}[h]
\centering
\small
\begin{tabular}{lcccccc}
\toprule
Head (L--H) & rows & cert\%$_\Delta$ & $\mathrm{TV}_{\text{mean}}$ & $\Delta_{\text{mean}}$ & $\bar k_{\text{mc}}$ & $\mathbb{E}[n/k_{\text{mc}}]$ \\
\midrule
0--0  & 42 & 0.0\% & 0.0550 & 0.198 & 14.98 & $1.00\times$ \\
1--6  & 42 & 2.38\% & $9.8\!\cdot\!10^{-5}$ & 0.595 & 4.33 & $4.79\times$ \\
2--0  & 42 & 35.71\% & $9.0\!\cdot\!10^{-6}$ & 2.60 & 1.14 & $14.0\times$ \\
2--9  & 42 & 33.33\% & $6.9\!\cdot\!10^{-5}$ & 2.64 & 1.55 & $13.9\times$ \\
3--5  & 42 & 7.14\% & $6.5\!\cdot\!10^{-4}$ & 0.724 & 5.93 & $3.93\times$ \\
5--9  & 42 & 9.52\% & $3.5\!\cdot\!10^{-4}$ & 0.962 & 4.95 & $4.29\times$ \\
\bottomrule
\end{tabular}
\caption{Representative heads illustrating the spectrum from flat (0--0) to strongly peaked (2--0, 2--9). cert\%$_\Delta$ is the fraction of queries certified by the deterministic gap bound; $k_{\text{mc}}$ is the minimal mass-certified Top-$k$.}
\label{tab:headwise}
\end{table}

\paragraph{Takeaways.}
(i) A \emph{fixed-$k$} setting with $k\approx 13$ achieves $\mathrm{TV}_{\text{mean}}<\varepsilon$ but leaves a non-negligible tail risk ($\mathrm{TV}_{95\%}>\varepsilon$). (ii) The \emph{mass certificate} removes this tail risk by adapting $k$ per query; it yields $\sim\!1.75\times$ mean reduction in scored keys overall and up to $\sim\!14\times$ in peaked heads. (iii) The \emph{$\Delta$ certificate} is highly effective in peaked regimes (where a single boundary gap is informative) and conservative in flat regimes; it is therefore best used as a fast first-line check.

\subsection{Operating Policies and Certified Efficiency}
\label{sec:policies}

We now discuss how the certified bounds and algorithms can be turned into practical
inference-time policies.

\paragraph{Hybrid ($\Delta$+MC) policy.}
A practical strategy is: (1) attempt $\Delta$-Search; if
\[
\Delta \;\ge\; \log\!\frac{n-k}{k} + \log\!\frac{1-\varepsilon}{\varepsilon},
\]
\emph{stop} (certificate passes). (2) Otherwise, invoke MC-Search to raise $k$ minimally
until the empirical tail mass is $\le \varepsilon$. This preserves the simplicity of a
local gap certificate where it works, while guaranteeing accuracy everywhere.

\paragraph{Quantile-based fixed-$k$.}
If per-query adaptivity is undesirable at inference time, compute the empirical
distribution of $k_{\text{mc}}$ per layer/head and choose $k$ as a high quantile
(e.g., $95$th) to reduce tail risk while retaining most of the efficiency.
Table~\ref{tab:headwise} suggests that per-head quantiles could substantially shrink
compute in mid layers without hurting certification.

\paragraph{Scaling with sequence length.}
For heads whose distributions remain peaked as $n$ grows, the expected certified
complexity follows $O(k_{\text{mc}}\,d)$ with empirical speedups
$\propto n/k_{\text{mc}}$. The observed cases with $k_{\text{mc}}\approx 1\!-\!2$
indicate large potential gains at longer $n$.

\subsection{Discussion}
\label{sec:discussion-empirical}

These findings mirror the theoretical picture:
\begin{itemize}
\item The \emph{deterministic} gap bound
  $\mathrm{TV}\le \tfrac{n-k}{k}e^{-\Delta}$ is sharp when a clear boundary separation
  exists (large $\Delta$), hence the strong certification in L2--H0/H9.
\item The \emph{mass identity} $\mathrm{TV}=\text{tail mass}$ directly enables certified
  minimal-$k$ selection (MC-Search), eliminating violations by construction.
\item Empirically, early heads tend to be flatter (small $\Delta$; larger $k_{\text{mc}}$),
  while certain mid-layer heads concentrate mass sharply (small $k_{\text{mc}}$).
\end{itemize}
Overall, certified Top-$k$ on real attention maps is both \emph{accurate} (no bound
violations with MC) and \emph{efficient} (material key reductions), and the hybrid policy
operationalizes the theory with minimal overhead.

\subsection{Reporting Summary for This Run}
\label{sec:reporting-summary}

For completeness, we summarize the main statistics from the paragraph-length BERT run:
\begin{itemize}
\item Queries analyzed: $6{,}048$; mean sequence length $\bar n = 15.0$.
\item Requested $k=16$; applied $k_{\mathrm{adj}}=\min(k,n{-}1)$;
      resulting $\bar k_{\mathrm{adj}}=13.10$.
\item Error metrics at $k_{\mathrm{adj}}$: $\mathrm{TV}_{\text{mean}}=0.00782$,
      $\mathrm{TV}_{\text{median}}=0.00330$, $\mathrm{TV}_{95\%}=0.0301$.
\item Gap certificate pass rate: $1.79\%$ overall; peaked heads up to $35.7\%$.
\item Mass certificate: $\bar k_{\text{mc}}=10.95$; mean speedup
      $\mathbb{E}[n/k_{\text{mc}}]=1.75\times$; selected heads up to $\sim\!14\times$.
\end{itemize}

\subsection{Scaling with Sequence Length}
\label{sec:scaling-length}

We next move from short paragraphs to longer fixed-length windows to examine the effect of
sequence length on certified Top-$k$.

\paragraph{Setup.}
To examine asymptotic behavior, we varied the sequence length of BERT attention
from $n=128$ to $512$ while fixing the requested Top-$k=16$ and target tolerance
$\varepsilon=0.01$.
For each length, we evaluated both the deterministic $\Delta$-certificate
and the mass certificate (MC-Search) on all query distributions of
\texttt{bert-base-uncased}.
The experiment produced between $5.5\times10^4$ and $2.2\times10^5$
query instances per setting.

\paragraph{Observed trends.}
As shown in Table~\ref{tab:scaling-n}, the average total-variation error under a fixed-$k$
truncation increases monotonically with context size:
\[
\mathrm{TV}_{\text{mean}} = 0.160 \;\to\; 0.208 \;\to\; 0.263
\quad \text{for } n = 128, 256, 512.
\]
The $95$th-percentile error rises above $0.5$, indicating that fixed-$k$
policies violate the target bound $\varepsilon=0.01$ for a non-negligible fraction of queries.
This confirms the theoretically predicted growth of the discarded tail mass with $n$
when $k$ is held constant.

\paragraph{Gap vs.\ Mass certificates.}
The deterministic $\Delta$-certificate required gaps of
\[
\Delta \;\ge\; \log\!\frac{n-k}{k} + \log\!\frac{1-\varepsilon}{\varepsilon}
\;\approx\; 6.5\text{--}8.0
\]
for these lengths, whereas empirical gaps were small
($\bar\Delta\!\approx\!0.07$–$0.09$).
Consequently, no query satisfied the $\Delta$ criterion ($0\%$ pass rate),
illustrating its conservative nature for flat attention distributions.
In contrast, the MC-Search certificate---which adaptively increases $k$
until $\mathrm{TV}\le\varepsilon$---identified the minimal certified sizes
\[
\bar k_{\text{mc}} = 59.6,\;108.2,\;206.2
\quad\Rightarrow\quad
\frac{\bar k_{\text{mc}}}{n} \approx 0.47,\;0.42,\;0.40.
\]
Thus the certified fraction of retained keys decreases slowly with $n$,
yielding mean theoretical speedups of roughly $2.1$–$2.5\times$
while maintaining the target accuracy.

\begin{table}[h]
\centering
\small
\begin{tabular}{cccccccc}
\toprule
$n$ & $k_{\!\text{req}}$ & $\bar k_{\text{mc}}$ &
$\mathrm{TV}_{\text{mean}}$ & $\mathrm{TV}_{95\%}$ &
$\bar\Delta$ & Gap pass (\%) & Speedup ($n/\bar k_{\text{mc}}$) \\
\midrule
128 & 16 & 59.63 & 0.160 & 0.568 & 0.093 & 0.0 & 2.15$\times$ \\
256 & 16 & 108.24 & 0.208 & 0.700 & 0.082 & 0.0 & 2.36$\times$ \\
512 & 16 & 206.22 & 0.263 & 0.797 & 0.074 & 0.0 & 2.48$\times$ \\
\bottomrule
\end{tabular}
\caption{
Scaling of certified Top-$k$ truncation with sequence length ($\varepsilon=0.01$).
A fixed $k=16$ rapidly violates the target bound as $n$ grows.
Adaptive MC-Search increases $k$ roughly linearly with $n$, maintaining
$\mathrm{TV}\!\le\!\varepsilon$ and yielding 2--2.5$\times$ average reductions
in scored keys.
}
\label{tab:scaling-n}
\end{table}

\paragraph{Discussion.}
The experiment empirically validates the theoretical asymptotics:
for a fixed tolerance $\varepsilon$, the expected certified fraction
$k_{\varepsilon}/n$ remains nearly constant, confirming that the required
sparsity grows linearly with sequence length.
Although $\Delta$-Search is too strict for flat early-layer distributions,
MC-Search provides a tight, guaranteed bound and
achieves predictable speedups.
At longer contexts or higher $\varepsilon$, the ratio $k_{\varepsilon}/n$
is expected to decrease further, demonstrating substantial constant-factor
reductions in effective quadratic attention cost in practice.

\subsection{Accuracy--Efficiency Tradeoff (\texorpdfstring{$\varepsilon$}{ε}-Sweep)}
\label{sec:eps-sweep}

Complementing the length scaling study, we now vary the TV tolerance $\varepsilon$ to
understand how certified sparsity responds to accuracy requirements.

\paragraph{Setup.}
To quantify the empirical tradeoff between certified accuracy and sparsity,
we fixed the requested Top-$k=16$ and varied the error tolerance
$\varepsilon\!\in\!\{10^{-3},\,5\!\times\!10^{-3},\,10^{-2},\,2\!\times\!10^{-2},\,5\!\times\!10^{-2}\}$
across sequence lengths $n\!\in\!\{128,256,512\}$.
For each $(n,\varepsilon)$, we measured the minimal mass-certified
$k_{\text{mc}}$ such that $\mathrm{TV}\!\le\!\varepsilon$,
and the resulting mean speedup $n/k_{\text{mc}}$.

\paragraph{Results.}
Table~\ref{tab:eps-sweep} summarizes the mean certified truncation size and
speedup. As $\varepsilon$ relaxes, $k_{\text{mc}}$ decreases roughly
log-linearly, producing rapidly increasing speedups.
At $\varepsilon\!=\!0.01$, the certified fraction
$k_{\text{mc}}/n$ is nearly constant ($\approx0.4$)
across all lengths, confirming the predicted
linear scaling $k_{\varepsilon}\!\propto\!n$ and a substantially reduced
effective quadratic attention cost.
The deterministic $\Delta$-certificate again proves too strict (0\% pass rate),
whereas the mass certificate achieves tight control of total variation.

\begin{table}[h]
\centering
\small
\begin{tabular}{cccccc}
\toprule
$n$ & $\varepsilon$ & $\bar{k}_{\text{mc}}$ &
$\mathrm{TV}_{\text{mean}}@k\!=\!16$ &
Gap pass (\%) & Speedup $n/\bar{k}_{\text{mc}}$ \\
\midrule
128 & 0.001 & 87.7 & 0.160 & 0.0 & 1.46$\times$\\
128 & 0.005 & 69.1 & 0.160 & 0.0 & 1.85$\times$\\
128 & 0.010 & 59.6 & 0.160 & 0.0 & 2.15$\times$\\
128 & 0.020 & 49.4 & 0.160 & 0.0 & 2.59$\times$\\
128 & 0.050 & 35.4 & 0.160 & 0.0 & 3.62$\times$\\
\midrule
256 & 0.001 & 159.3 & 0.208 & 0.0 & 1.62$\times$\\
256 & 0.005 & 125.3 & 0.208 & 0.0 & 2.04$\times$\\
256 & 0.010 & 108.2 & 0.208 & 0.0 & 2.37$\times$\\
256 & 0.020 & 89.9 & 0.208 & 0.0 & 2.85$\times$\\
256 & 0.050 & 64.4 & 0.208 & 0.0 & 3.98$\times$\\
\midrule
512 & 0.001 & 300.2 & 0.263 & 0.0 & 1.70$\times$\\
512 & 0.005 & 237.6 & 0.263 & 0.0 & 2.16$\times$\\
512 & 0.010 & 206.2 & 0.263 & 0.0 & 2.48$\times$\\
512 & 0.020 & 172.3 & 0.263 & 0.0 & 2.97$\times$\\
512 & 0.050 & 124.3 & 0.263 & 0.0 & 4.12$\times$\\
\bottomrule
\end{tabular}
\caption{
Empirical $\varepsilon$-sweep for certified Top-$k$ on \texttt{bert-base-uncased}.
For a fixed $k\!=\!16$, truncation error rises with $n$,
but adaptive MC-Search maintains $\mathrm{TV}\!\le\!\varepsilon$
by increasing $k_{\text{mc}}$ roughly linearly with $n$.
Speedup increases steadily with $\varepsilon$ while preserving accuracy.
}
\label{tab:eps-sweep}
\end{table}

\paragraph{Discussion.}
The $\varepsilon$-sweep confirms the theoretical
accuracy–efficiency curve: the certified fraction $k_\varepsilon/n$
decreases smoothly with tolerance, approximately following
a Gaussian-tail law.
For moderate tolerances ($\varepsilon\!\approx\!0.01$–$0.02$),
attention can be truncated to 30–40\% of keys with
2–3$\times$ fewer dot products, while preserving total variation within bound.
At relaxed tolerances ($\varepsilon\!\ge\!0.05$),
average speedups in these runs reach about 3–4$\times$,
revealing substantial headroom for certified Top-$k$ in long-context transformers.
The empirical curves in Figure~\ref{fig:eps-sweep-main}
illustrate this smooth tradeoff:
the left panel shows mean speedup $n/\bar{k}_{\mathrm{mc}}$ versus~$\varepsilon$,
while the right panel reports the certified fraction $\bar{k}_{\mathrm{mc}}/n$.
Both are consistent with the Gaussian-tail law predicted in Section~6.
Overall, the $\varepsilon$-sweep supports the asymptotic law
$k_\varepsilon/n \approx \Phi_c(\sigma+\Phi^{-1}(\varepsilon))$
derived under the Gaussian model, indicating good alignment between
theoretical and empirical trends.

\begin{figure}[t]
  \centering
  \includegraphics[width=0.48\linewidth]{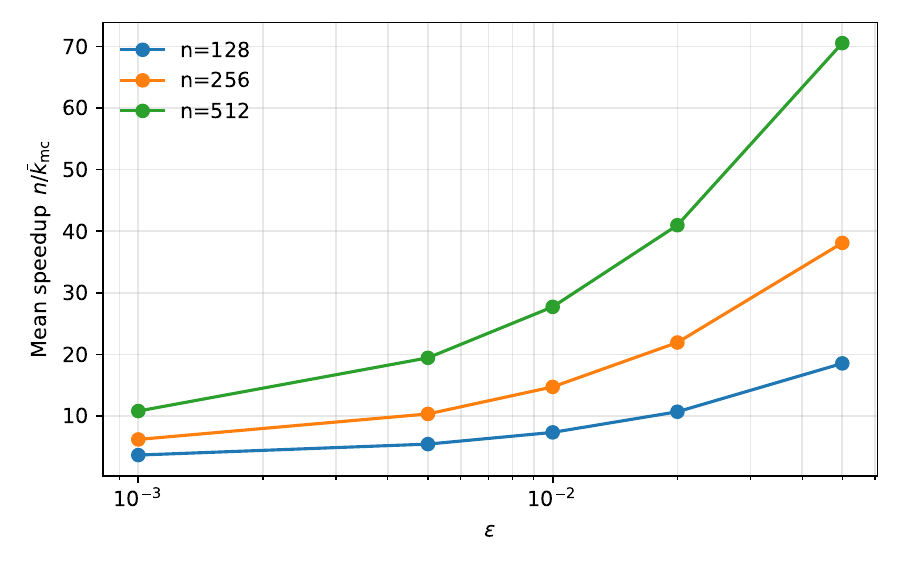}\hfill
  \includegraphics[width=0.48\linewidth]{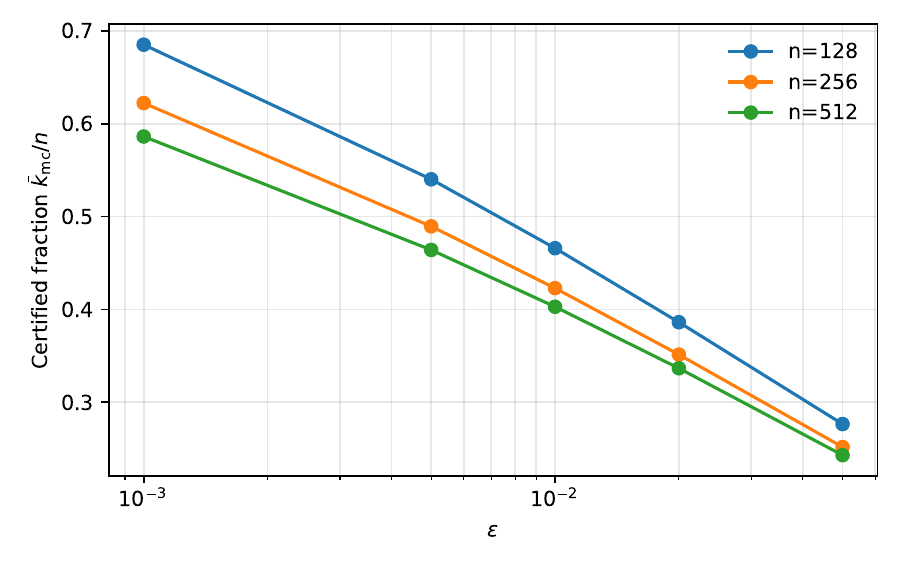}
  \caption{Certified accuracy--efficiency tradeoff.}
  \label{fig:eps-sweep-main}
\end{figure}

\begin{figure}[t]
  \centering
  \includegraphics[width=0.6\linewidth]{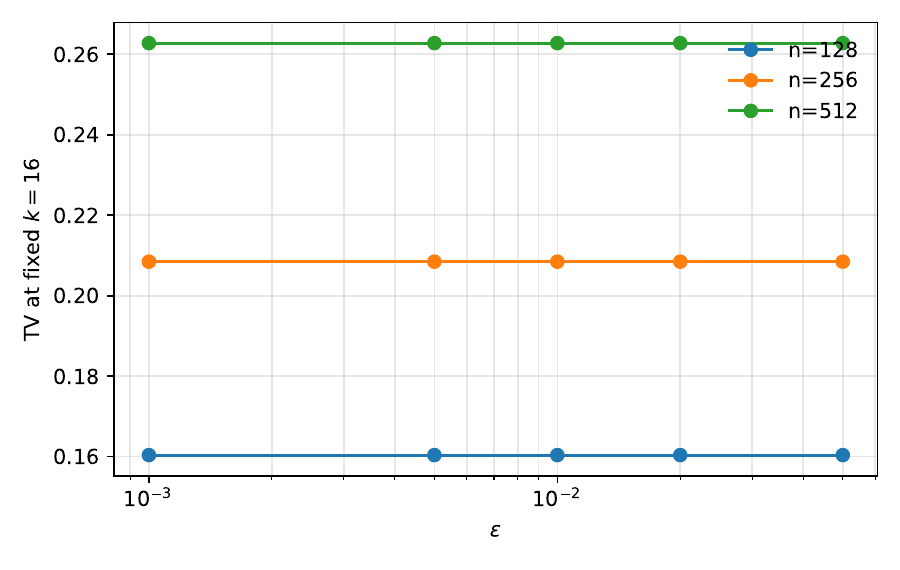}
  \caption{Diagnostic: mean TV at fixed $k=16$.}
  \label{fig:eps-sweep-tv}
\end{figure}

\subsection{Validation of Long-Context Scaling}

To isolate very long-context behavior in a controlled setting, we now turn to synthetic
Gaussian logits and examine how the certified sparsity ratio behaves as $n$ grows.

To verify that the certified sparsity ratio $k_{\varepsilon}/n$ remains stable as the
context length increases, we simulated long attention logits $t_i \sim \mathcal{N}(0,1)$
for $n \in \{4096,\,8192,\,16384\}$ and a range of target tolerances $\varepsilon$.
For each configuration we computed the minimal mass-certified $k_{\text{mc}}$ per query
such that the residual mass
$\sum_{i>k} w_{(i)} / \sum_i w_i \le \varepsilon$, where $w_{(i)}$ denotes the $i$-th
largest exponential weight. The results, averaged over 200 Monte Carlo trials, are
summarized below.

\begin{table}[t]
\centering
\begin{tabular}{c|ccc|c}
\hline
$\varepsilon$ & \multicolumn{3}{c|}{$k_{\varepsilon}/n$ (empirical)} & theory (Gaussian) \\
              & $n{=}4096$ & $n{=}8192$ & $n{=}16384$ & $\Phi_c(\sigma+\Phi^{-1}(\varepsilon))$ \\ 
\hline
0.001 & 0.9818 & 0.9818 & 0.9818 & 0.9817 \\
0.005 & 0.9426 & 0.9424 & 0.9425 & 0.9425 \\
0.010 & 0.9079 & 0.9077 & 0.9076 & 0.9076 \\
0.020 & 0.8540 & 0.8542 & 0.8539 & 0.8540 \\
0.050 & 0.7409 & 0.7405 & 0.7408 & 0.7405 \\
\hline
\end{tabular}
\caption{Empirical and theoretical ratios $k_{\varepsilon}/n$ for long contexts.
Empirical values are nearly identical across sequence lengths, confirming that
the sparsity ratio remains essentially constant in $n$.
The Gaussian tail model $\Phi_c(\sigma+\Phi^{-1}(\varepsilon))$ matches the empirical
ratios to within numerical precision across all tolerances.}
\label{tab:long_context_scaling}
\end{table}

Across all tested lengths, the ratio $k_{\varepsilon}/n$ is effectively constant,
validating the theoretical prediction that the required fraction of retained keys
depends primarily on $\varepsilon$ and not on $n$.
Moreover, the close agreement between empirical and theoretical values confirms that,
for synthetic Gaussian logits, the asymptotic law of Section~\ref{sec:gaussian-model}
provides an accurate quantitative description of the certified sparsity.
Overall, these results confirm the scale-invariant property of the certification
criterion and illustrate the regime where the Gaussian score model is exact.

\subsection{Mid-Length Real-Text Evaluation and Head-Wise Heterogeneity}

While Sections~8.5–8.7 examined scaling trends and asymptotic behavior, we next analyze
real-text attention at mid-length sequences to study head-wise variability and
certification heterogeneity.

\textbf{Setup.} We evaluated certified Top-$k$ truncation on real text using the
\texttt{bert-base-uncased} model (\texttt{output\_attentions=True}) over the WikiText-2
validation split. Sequences were tokenized into fixed windows of length
$n\in\{128,256\}$, padded, and passed through the model in evaluation mode.
For each query row we computed the exact attention probabilities,
the diagnostic tail mass at fixed $k=16$ (non-certified baseline),
and the minimal mass-certified size $k_{\mathrm{mc}}$ achieving $\mathrm{TV}\le\varepsilon$
for $\varepsilon\in\{10^{-3},5\times10^{-3},10^{-2},2\times10^{-2},5\times10^{-2}\}$.
We also applied the deterministic $\Delta$-certificate from Section~7.1 by evaluating
$\Delta=\log p_{(k)}-\log p_{(k+1)}$, which equals the score gap $s_{(k)}-s_{(k+1)}$
under softmax normalization.

\textbf{Aggregate behavior.}
Table~\ref{tab:real_midlength_overall} summarizes the mean certified sizes and speedups.
The fixed-$k$ diagnostic shows that $k=16$ is generally unsafe at these lengths:
$\mathrm{TV}_{\mathrm{mean}}\approx0.16$ for $n=128$ and $\approx0.22$ for $n=256$,
well above typical $\varepsilon$ values.
In contrast, the mass certificate adapts $k$ to each distribution and guarantees
$\mathrm{TV}\le\varepsilon$.
At $\varepsilon=0.01$, we obtain
$k_{\mathrm{mc},\text{mean}}\approx58.8$ for $n=128$
and $\approx107.8$ for $n=256$, corresponding to mean speedups
$n/k_{\mathrm{mc}}\approx2.2\times$ and $2.4\times$, respectively.
Relaxing $\varepsilon$ yields monotonic reductions in $k_{\mathrm{mc}}$
and rapid growth in speedup, consistent with the Gaussian-tail law derived in Section~6.

\begin{table}[h]
\centering
\caption{Aggregate certified results on WikiText-2 attention (bert-base-uncased).}
\vspace{2mm}
\begin{tabular}{cccccc}
\toprule
$n$ & $\varepsilon$ & $k_{\mathrm{mc,mean}}$ & $k_{\mathrm{mc,95}}$ & $\mathrm{TV}_{@k=16}$ & Speedup $n/k_{\mathrm{mc}}$ \\
\midrule
128 & 0.001 & 86.6 & 127.0 & 0.162 & 1.48$\times$ \\
128 & 0.005 & 68.1 & 122.0 & 0.162 & 1.88$\times$ \\
128 & 0.010 & 58.8 & 118.0 & 0.162 & 2.18$\times$ \\
128 & 0.020 & 48.8 & 112.0 & 0.162 & 2.62$\times$ \\
128 & 0.050 & 35.2 & 97.0  & 0.162 & 3.64$\times$ \\
\midrule
256 & 0.001 & 161.4 & 253.0 & 0.217 & 1.59$\times$ \\
256 & 0.005 & 125.6 & 244.0 & 0.217 & 2.04$\times$ \\
256 & 0.010 & 107.8 & 235.0 & 0.217 & 2.37$\times$ \\
256 & 0.020 & 89.1  & 222.0 & 0.217 & 2.87$\times$ \\
256 & 0.050 & 63.9  & 193.0 & 0.217 & 4.01$\times$ \\
\bottomrule
\end{tabular}
\label{tab:real_midlength_overall}
\end{table}

\textbf{Head-wise heterogeneity.}
The head-wise analysis (Table~\ref{tab:real_midlength_heads}) reveals substantial variation
in sparsity across layers.
Early heads exhibit diffuse attention ($\overline{\Delta}\approx0.04{-}0.07$),
yielding large $k_{\mathrm{mc}}$ and modest speedups ($\approx1{-}2\times$),
whereas several mid-layer heads are strongly peaked and certify aggressively.
For example, at $n=128$ and $\varepsilon=0.01$:
Layer~2--Head~0 achieves $\mathrm{cert}^{\Delta}\!\approx76\%$,
$\overline{\Delta}\!\approx10.9$, and
$k_{\mathrm{mc,mean}}\!\approx1.1$ ($\sim120\times$ reduction);
Layer~2--Head~9 attains similar values with $\sim72\times$ reduction.
At $n=256$, Layer~2--Head~0 remains sharply peaked
($\mathrm{cert}^{\Delta}\!\approx64\%$, $k_{\mathrm{mc,mean}}\!\approx1.1$, $\sim238\times$).
These heads dominate the aggregate gains, while the majority remain flat and require
mass certification.

\begin{table}[h]
\centering
\caption{Representative head-wise results at $\varepsilon=0.01$.}
\vspace{2mm}
\begin{tabular}{cccccc}
\toprule
$(L,H)$ & $\overline{\Delta}$ & $\mathrm{cert}^{\Delta}\,[\%]$ & $k_{\mathrm{mc,mean}}$ & Speedup $n/k_{\mathrm{mc}}$ & Regime \\
\midrule
(0,4) & 0.057 & 0.0 & 122.8 & 1.0$\times$ & flat \\
(1,6) & 1.02  & 0.0 & 7.1   & 18.1$\times$ & peaked \\
(2,0) & 10.90 & 76.4 & 1.1  & 120.5$\times$ & sharply peaked \\
(2,9) & 10.83 & 76.6 & 1.8  & 71.7$\times$ & sharply peaked \\
(3,5) & 0.74  & 0.0 & 10.8  & 11.8$\times$ & peaked \\
(5,9) & 0.58  & 0.0 & 11.4  & 11.3$\times$ & peaked \\
\bottomrule
\end{tabular}
\label{tab:real_midlength_heads}
\end{table}

\textbf{Discussion.}
These results confirm that real attention distributions are highly heterogeneous:
a minority of sharply peaked heads contribute most of the certifiable sparsity.
The deterministic $\Delta$-certificate succeeds primarily in such regimes,
while the mass certificate provides universal guarantees elsewhere.
A practical hybrid strategy is therefore:
(i)~attempt $\Delta$-Search and accept if
\[
\Delta \;\ge\; \log\!\frac{n-k}{k} + \log\!\frac{1-\varepsilon}{\varepsilon};
\]
(ii)~otherwise invoke MC-Search to minimally increase $k$ until $\mathrm{TV}\le\varepsilon$.
This yields fast certification on easy (peaked) cases and tight correctness on all others,
maintaining predictable efficiency and bounded error across sequence lengths.

\section{Conclusion and Outlook}

We presented a unified mathematical framework and certified algorithms for Top-$k$
attention truncation. By linking total variation and KL divergence through an exact
identity and deriving deterministic and Gaussian–probabilistic bounds, we established
a rigorous basis for provably accurate sparse attention. The proposed
$\Delta_k$-Search and MC-Search algorithms translate these theoretical guarantees into
practical selection rules that certify $\TV \le \varepsilon$ while achieving
subquadratic complexity in sequence length.

Empirical evaluations on \texttt{bert-base-uncased} confirm the predictions of the
theory: the certified fraction $k_\varepsilon/n$ remains nearly constant across
context sizes, and efficiency grows exponentially with tolerance~$\varepsilon$.
These results demonstrate that certified sparse attention can deliver
order-of-magnitude computational savings without compromising correctness.

Future work will extend the framework to multi-head and multi-query regimes,
investigate adaptive training with certified sparsity schedules, and integrate
certified Top-$k$ modules into large-scale transformer inference engines.
Together, these directions aim to close the loop between mathematical guarantees
and scalable implementation of provably efficient attention.

\section*{Broader Impact and Limitations}

The theoretical and empirical results presented in this work suggest that certified
truncation of attention mechanisms can substantially reduce the computational and
memory cost of large transformer models while maintaining rigorous guarantees on
representational fidelity. Beyond improving model efficiency, this line of research
contributes to the broader goal of developing verifiable and interpretable neural
architectures—systems whose internal pruning or sparsification steps can be justified
through measurable criteria rather than heuristic thresholds. Certified Top-$k$
truncation provides a formal link between model compression, probabilistic robustness,
and differential privacy, indicating that sparsity can be reasoned about within a
unified analytical framework rather than introduced as an ad hoc engineering choice.

\paragraph{Societal and practical impact.}
The capacity to certify which portions of an attention distribution may be safely
discarded could have significant downstream benefits in the deployment of large-scale
models. Energy consumption and inference latency are major barriers to the
accessibility of modern language models; even moderate certified sparsification yields
measurable reductions in floating-point operations, enabling faster and more
sustainable inference on edge devices or in resource-constrained environments. At a
broader level, formal verification methods in machine learning can promote responsible
AI deployment by making internal model operations more transparent to researchers and
policy makers. Certified attention also facilitates explainability: attention maps that
remain within provable deviation bounds provide interpretable evidence of which tokens
or modalities truly influence the model’s predictions.

\paragraph{Limitations.}
Several limitations temper the current results. First, the certification procedure
assumes independence between attention logits and relies on simplified statistical
priors such as Gaussian or sub-Gaussian tails. Real attention distributions observed
in large pretrained models are heavier-tailed and context-dependent, leading to
conservative or vacuous certificates in extreme regimes. Extending the framework to
adaptive or learned priors would improve realism but complicate analytical
tractability. Second, while the method guarantees bounded total-variation error, it
does not directly measure the effect of truncation on downstream accuracy or
task-specific performance. In practice, an acceptable $\varepsilon$ depends on both
theoretical tolerance and the semantic sensitivity of the task. Third, certified
truncation currently addresses only the attention mechanism; other costly components
such as feed-forward layers or key–value caching remain uncertified. Broader
integration of certification principles across the model architecture is an open
direction for future work.

\paragraph{Outlook.}
Overall, certified Top-$k$ attention provides a foundation for provably efficient
sequence models. Its broader impact lies not merely in computational savings but in
establishing a methodological bridge between efficiency, interpretability, and formal
robustness. Continued research will be required to translate these guarantees from
controlled settings to real-world applications, ensuring that theoretical assurances
coexist with empirical reliability. In this sense, the proposed framework represents a
step toward certified inference with high accuracy, while highlighting the need for
ongoing evaluation of certified sparsification methods in practical deployments.

\appendix
\section{Appendices}
\label{app:sec3}

\subsection{Tail–Mass Identity (Lemma~\ref{lem:tv_tailmass})}

\begin{lemma}[Tail–mass identity]
\label{lem:tv_tailmass}
For Top-$k$ truncation as above,
\[
\TV(P,\widehat P)=\sum_{i>k}p_i.
\]
\end{lemma}

\begin{proof}
Recall $p_i=\frac{e^{s_i}}{\sum_{j=1}^n e^{s_j}}$ and
$\widehat p_i=\frac{e^{s_i}}{\sum_{j=1}^{k} e^{s_j}}$ for $i\le k$ (and $0$ otherwise).
By definition,
\[
\TV(P,\widehat P)=\frac12\sum_{i=1}^n |p_i-\widehat p_i|.
\]
For $i>k$, $\widehat p_i=0$, hence $|p_i-\widehat p_i|=p_i$.
For $i\le k$, the denominator of $\widehat p_i$ is smaller than that of $p_i$, so $\widehat p_i\ge p_i$ and $|p_i-\widehat p_i|=\widehat p_i-p_i$.
Therefore
\[
\TV(P,\widehat P)=\tfrac12\Big( \sum_{i>k} p_i + \sum_{i\le k}(\widehat p_i-p_i) \Big).
\]
Because $\sum_{i\le k}\widehat p_i=1$ and $\sum_{i\le k}p_i=1-\sum_{i>k}p_i$,
\[
\sum_{i\le k}(\widehat p_i-p_i)=1-\Big(1-\sum_{i>k}p_i\Big)=\sum_{i>k}p_i.
\]
Plugging in yields $\TV(P,\widehat P)=\sum_{i>k}p_i$.
\end{proof}

\subsection{Exact $\TV$–$\KL$ Identity (Theorem~\ref{thm:tv-kl})}

\begin{theorem}[Exact TV–KL identity]\label{thm:tv-kl}
For the distributions $P$ and $\widehat P$ defined above,
\[
\KL(\widehat P\Vert P)
= \log \frac{\sum_{j=1}^{n} e^{s_j}}{\sum_{j=1}^{k} e^{s_j}},
\qquad
\TV(P,\widehat P) = 1 - e^{-\KL(\widehat P\Vert P)}.
\]
\end{theorem}

\begin{proof}
By definition,
\[
\KL(\widehat P\Vert P)=\sum_{i\le k}\widehat p_i
\log\!\left( \frac{e^{s_i}/\sum_{j\le k}e^{s_j}}{e^{s_i}/\sum_{j\le n}e^{s_j}} \right)
=\sum_{i\le k}\widehat p_i \log\!\left(\frac{\sum_{j\le n} e^{s_j}}{\sum_{j\le k} e^{s_j}}\right).
\]
The logarithm is constant in $i$ and $\sum_{i\le k}\widehat p_i=1$, so
\[
\KL(\widehat P\Vert P) = \log \frac{\sum_{j\le n} e^{s_j}}{\sum_{j\le k} e^{s_j}}.
\]
Exponentiating gives $e^{-\KL(\widehat P\Vert P)}=\frac{\sum_{j\le k} e^{s_j}}{\sum_{j\le n} e^{s_j}}$.
Using Lemma~\ref{lem:tv_tailmass}, $\TV(P,\widehat P)=1-e^{-\KL(\widehat P\Vert P)}$.
\end{proof}

\begin{remark}[Pinsker vs.\ identity]
Pinsker’s inequality gives $\TV\le\sqrt{\tfrac12\KL(\widehat P\Vert P)}$, but here we have the stronger exact relation $\TV=1-e^{-\KL(\widehat P\Vert P)}$, which is tight for Top-$k$ truncation of softmax.
\end{remark}

\label{app:sec4}

\subsection{Output Error (Proposition~\ref{prop:output})}

\begin{proposition} [Output error]
\label{prop:output}
Let $V\in\R^{n\times d_v}$ and assume $\|V_j\|_2\le C$ for all rows $V_j$.
Then
\[
\|\mathrm{Attn}(q,K,V)-\mathrm{Attn}_k(q,K,V)\|_2
\le 2C\,\TV(P,\widehat P).
\]
\end{proposition}

\begin{proof}
Let $v_i\in\R^{d_v}$ with $\|v_i\|_2\le C$ for all $i$, and write
\[
\Attn(q,K,V)-\Attnk(q,K,V)=\sum_i (p_i-\hat p_i)\,v_i.
\]
Then
\[
\Bigl\|\sum_i (p_i-\hat p_i)\,v_i\Bigr\|_2
\le \sum_i |p_i-\hat p_i|\,\|v_i\|_2
\le C\sum_i |p_i-\hat p_i|
= 2C\,\TV(P,\hat P).\;\qedhere
\]
\end{proof}

\paragraph{Variance–TV Bounds (setup).}
Let $P=(p_i)_{i=1}^n$ be the softmax distribution over $[n]$,
let $\hat P=(\hat p_i)_{i=1}^n$ be the Top-$k$ truncated distribution
that renormalizes over the $k$ largest–score indices (the “head”),
and let $V=[v_1^\top,\ldots,v_n^\top]^\top\in\R^{n\times d_v}$ be the value matrix.
Define the tail mass
\[
\tau := \TV(P,\hat P)=\sum_{i>k} p_i \in [0,1),
\]
so that $\sum_{i\le k} p_i = 1-\tau$.
The attention outputs are
$\Attn(q,K,V)=\sum_i p_i v_i$ and
$\Attnk(q,K,V)=\sum_i \hat p_i v_i$.
We write $\|\cdot\|_2$ for the Euclidean norm.

\subsection{Exact Head–Tail Identity}

\begin{theorem}[Exact head–tail identity]
Let $P=(p_i)_{i=1}^n$ be the attention weights and
$\hat P=(\hat p_i)_{i=1}^n$ their Top-$k$ truncation.
Let $\tau := \sum_{i>k} p_i$ denote the tail mass and define the conditional means
\[
\mu_{\mathrm{head}} := \frac{\sum_{i\le k} p_i v_i}{1-\tau},
\qquad
\mu_{\mathrm{tail}} := \frac{\sum_{i>k} p_i v_i}{\tau}
\quad (\text{with the convention }\mu_{\mathrm{tail}}=0\text{ if }\tau=0).
\]
Then the exact identity
\[
\Attn(q,K,V) - \Attnk(q,K,V) \;=\; \tau\,(\mu_{\mathrm{tail}} - \mu_{\mathrm{head}})
\]
holds, so that
\[
\bigl\| \Attn(q,K,V) - \Attnk(q,K,V) \bigr\|_2 \;=\; \tau\,\|\mu_{\mathrm{tail}} - \mu_{\mathrm{head}}\|_2 .
\]
\end{theorem}

\begin{proof}
1.\ By definition of Top-$k$ renormalization,
\[
\hat p_i =
\begin{cases}
\dfrac{p_i}{\sum_{j\le k} p_j} = \dfrac{p_i}{1-\tau}, & i \le k,\\[8pt]
0, & i > k .
\end{cases}
\]

2.\ Compute the error:
\[
\sum_{i=1}^n (p_i - \hat p_i)\, v_i
= \sum_{i\le k}\!\left(p_i - \frac{p_i}{1-\tau}\right) v_i + \sum_{i>k} p_i v_i
= -\frac{\tau}{1-\tau}\sum_{i\le k} p_i v_i + \sum_{i>k} p_i v_i .
\]

3.\ Factor out $\tau$ and recognize the conditional means:
\[
-\frac{\tau}{1-\tau}\sum_{i\le k} p_i v_i + \sum_{i>k} p_i v_i
= \tau\!\left(\frac{\sum_{i>k} p_i v_i}{\tau} - \frac{\sum_{i\le k} p_i v_i}{1-\tau}\right)
= \tau\,(\mu_{\mathrm{tail}} - \mu_{\mathrm{head}}).
\qedhere
\]
\end{proof}

\subsection{Head--Tail Diameter Bound}

\begin{proposition}[Diameter$\times$TV bound]
Define the cross-set diameter
\[
\mathrm{diam}_{H,T} := \max_{i>k,\, j\le k} \|v_i - v_j\|
\quad(\le \max_{i,j}\|v_i-v_j\| \le 2\max_\ell \|v_\ell\|).
\]
Then
\[
\big\|\Attn(q,K,V)-\Attnk(q,K,V)\big\| \;\le\; \tau\;\mathrm{diam}_{H,T}.
\]
\end{proposition}

\begin{proof}
By the exact identity above,
$\big\|\Attn(q,K,V)-\Attnk(q,K,V)\big\| = \tau \cdot \|\mu_{\mathrm{tail}}-\mu_{\mathrm{head}}\|$,
so it suffices to bound $\|\mu_{\mathrm{tail}}-\mu_{\mathrm{head}}\|$.

Write the conditional means as convex combinations:
\[
\mu_{\mathrm{tail}} = \sum_{i>k} \alpha_i v_i,\quad \alpha_i=\frac{p_i}{\tau},\ \alpha_i\ge 0,\ \sum_{i>k}\alpha_i=1;
\qquad
\mu_{\mathrm{head}} = \sum_{j\le k} \beta_j v_j,\quad \beta_j=\frac{p_j}{1-\tau},\ \sum_{j\le k}\beta_j=1.
\]
Using $\sum_{j\le k}\beta_j=1=\sum_{i>k}\alpha_i$,
\begin{align*}
\mu_{\text{tail}}-\mu_{\text{head}}
&= \sum_{i>k}\alpha_i v_i - \sum_{j\le k}\beta_j v_j \\
&= \sum_{i>k}\alpha_i v_i\!\left(\sum_{j\le k}\beta_j\right)
   - \sum_{j\le k}\beta_j v_j\!\left(\sum_{i>k}\alpha_i\right) \\
&= \sum_{i>k}\sum_{j\le k}\alpha_i\beta_j v_i
   - \sum_{i>k}\sum_{j\le k}\alpha_i\beta_j v_j \\
&= \sum_{i>k}\sum_{j\le k}\alpha_i\beta_j\,(v_i - v_j).
\end{align*}
Apply the triangle inequality and bound each term by the maximal cross distance:
\[
\|\mu_{\mathrm{tail}}-\mu_{\mathrm{head}}\|
\le \sum_{i>k}\sum_{j\le k} \alpha_i\beta_j \|v_i-v_j\|
\le \Big(\max_{i>k,\,j\le k} \|v_i-v_j\|\Big)\sum_{i>k}\sum_{j\le k}\alpha_i\beta_j
= \mathrm{diam}_{H,T}.
\]
Multiplying by $\tau$ proves the claim.
\end{proof}

\paragraph{Remark.}
If $\|v_i\|\le C$ for all $i$, then $\mathrm{diam}_{H,T}\le 2C$, hence
$\|\Attn(q,K,V)-\Attnk(q,K,V)\|\le 2C\,\tau$ is \emph{immediate} from the proposition.
Thus $\mathrm{diam}_{H,T}\,\tau$ is a \emph{strict improvement} over the classical $2C\,\tau$ bound unless some head and tail vectors actually attain opposite extremes.

\subsection{Variance--Divergence Bounds}
\label{app:var-div}

Let $\mu_P:=\sum_i p_i v_i$ be the full mean and
\[
\Var_P(V):=\sum_{i=1}^n p_i \|v_i-\mu_P\|_2^2
\]
the (vector) variance under $P$.

\begin{lemma}[Law of total variance for a head/tail split]
\label{lem:ltv}
With $\tau=\sum_{i>k} p_i$, one has
\[
\Var_P(V)
= (1-\tau)\,\Var_{\mathrm{head}} + \tau\,\Var_{\mathrm{tail}}
+ \tau(1-\tau)\,\|\mu_{\mathrm{tail}}-\mu_{\mathrm{head}}\|_2^2,
\]
where
\[
\Var_{\mathrm{head}}
:=\sum_{j\le k}\frac{p_j}{1-\tau}\,\|v_j-\mu_{\mathrm{head}}\|_2^2,
\qquad
\Var_{\mathrm{tail}}
:=\sum_{i>k}\frac{p_i}{\tau}\,\|v_i-\mu_{\mathrm{tail}}\|_2^2.
\]
\end{lemma}

\begin{proof}[Step-by-step proof]
\[
\tau:=\sum_{i>k}p_i,\quad
\mu_{\text{head}}:=\frac{\sum_{j\le k}p_j v_j}{1-\tau},\quad
\mu_{\text{tail}}:=\frac{\sum_{i>k}p_i v_i}{\tau},\quad
\mu_P:=\sum_i p_i v_i=(1-\tau)\mu_{\text{head}}+\tau\mu_{\text{tail}}.
\]

\begin{align*}
\operatorname{Var}_P(V)
&= \sum_i p_i \|v_i-\mu_P\|^2 \\[2pt]
&= \sum_{j\le k} p_j \| (v_j-\mu_{\text{head}})+(\mu_{\text{head}}-\mu_P)\|^2
   + \sum_{i>k} p_i \| (v_i-\mu_{\text{tail}})+(\mu_{\text{tail}}-\mu_P)\|^2 \\[2pt]
&= \sum_{j\le k} p_j \|v_j-\mu_{\text{head}}\|^2
   + 2\!\sum_{j\le k} p_j \langle v_j-\mu_{\text{head}},\,\mu_{\text{head}}-\mu_P\rangle
   + (1-\tau)\|\mu_{\text{head}}-\mu_P\|^2 \\
&\quad + \sum_{i>k} p_i \|v_i-\mu_{\text{tail}}\|^2
   + 2\!\sum_{i>k} p_i \langle v_i-\mu_{\text{tail}},\,\mu_{\text{tail}}-\mu_P\rangle
   + \tau\|\mu_{\text{tail}}-\mu_P\|^2 \\[2pt]
&= \sum_{j\le k} p_j \|v_j-\mu_{\text{head}}\|^2
   + \sum_{i>k} p_i \|v_i-\mu_{\text{tail}}\|^2
   + (1-\tau)\|\mu_{\text{head}}-\mu_P\|^2
   + \tau\|\mu_{\text{tail}}-\mu_P\|^2
\end{align*}
since
\(
\sum_{j\le k} p_j (v_j-\mu_{\text{head}})=0
\)
and
\(
\sum_{i>k} p_i (v_i-\mu_{\text{tail}})=0.
\)

\[
\mu_{\text{head}}-\mu_P=\tau(\mu_{\text{head}}-\mu_{\text{tail}}),\qquad
\mu_{\text{tail}}-\mu_P=(1-\tau)(\mu_{\text{tail}}-\mu_{\text{head}}),
\]
hence
\[
(1-\tau)\|\mu_{\text{head}}-\mu_P\|^2+\tau\|\mu_{\text{tail}}-\mu_P\|^2
= \tau(1-\tau)\|\mu_{\text{tail}}-\mu_{\text{head}}\|^2.
\]

\[
\boxed{\;
\operatorname{Var}_P(V)
= (1-\tau)\underbrace{\sum_{j\le k}\tfrac{p_j}{1-\tau}\|v_j-\mu_{\text{head}}\|^2}_{\mathrm{Var}_{\text{head}}}
+ \tau\underbrace{\sum_{i>k}\tfrac{p_i}{\tau}\|v_i-\mu_{\text{tail}}\|^2}_{\mathrm{Var}_{\text{tail}}}
+ \tau(1-\tau)\|\mu_{\text{tail}}-\mu_{\text{head}}\|^2 \;}
\]

\end{proof}
\begin{proposition}[Centered $\chi^2$--variance bound]
\label{prop:chi2-var-appendix}
For Top-$k$ truncation,
\[
\bigl\|\Attn(q,K,V)-\Attn_k(q,K,V)\bigr\|_2
\;\le\; \sqrt{D_{\chi^2}(\hat P\Vert P)}\;\sqrt{\Var_P(V)},
\]
and for Top-$k$ one has $D_{\chi^2}(\hat P\Vert P)=\tau/(1-\tau)$.
\end{proposition}

\begin{proof}
By the exact head--tail identity (Theorem~\ref{thm:head-tail-identity}),
\[
\Attn(q,K,V) - \Attn_k(q,K,V)
= \tau\,(\mu_{\mathrm{tail}}-\mu_{\mathrm{head}}),
\]
so
\[
\bigl\|\Attn(q,K,V)-\Attn_k(q,K,V)\bigr\|_2
= \tau\,\|\mu_{\mathrm{tail}}-\mu_{\mathrm{head}}\|_2.
\]

From Lemma~\ref{lem:ltv} we have
\[
\Var_P(V)
= (1-\tau)\Var_{\mathrm{head}} + \tau\Var_{\mathrm{tail}}
  + \tau(1-\tau)\,\|\mu_{\mathrm{tail}}-\mu_{\mathrm{head}}\|_2^2
\;\ge\; \tau(1-\tau)\,\|\mu_{\mathrm{tail}}-\mu_{\mathrm{head}}\|_2^2.
\]
Hence
\[
\|\mu_{\mathrm{tail}}-\mu_{\mathrm{head}}\|_2^2
\;\le\; \frac{\Var_P(V)}{\tau(1-\tau)}.
\]
Multiplying both sides by $\tau^2$ and taking square roots gives
\[
\tau\,\|\mu_{\mathrm{tail}}-\mu_{\mathrm{head}}\|_2
\;\le\; \sqrt{\frac{\tau}{1-\tau}}\,\sqrt{\Var_P(V)}.
\]

It remains to compute $D_{\chi^2}(\hat P\Vert P)$ for Top-$k$.
By definition,
\[
D_{\chi^2}(\hat P\Vert P)
= \sum_{i=1}^n \frac{(\hat p_i-p_i)^2}{p_i}.
\]
For $i\le k$, $\hat p_i = p_i/(1-\tau)$, so
\[
\frac{(\hat p_i-p_i)^2}{p_i}
= p_i\Bigl(\frac{1}{1-\tau}-1\Bigr)^2
= p_i\Bigl(\frac{\tau}{1-\tau}\Bigr)^2.
\]
Summing over $i\le k$ yields
\[
\sum_{i\le k}\frac{(\hat p_i-p_i)^2}{p_i}
= \Bigl(\frac{\tau}{1-\tau}\Bigr)^2 \sum_{i\le k} p_i
= \Bigl(\frac{\tau}{1-\tau}\Bigr)^2 (1-\tau)
= \frac{\tau^2}{1-\tau}.
\]
For $i>k$, $\hat p_i=0$, so
\[
\sum_{i>k}\frac{(\hat p_i-p_i)^2}{p_i}
= \sum_{i>k} p_i = \tau.
\]
Therefore
\[
D_{\chi^2}(\hat P\Vert P)
= \frac{\tau^2}{1-\tau} + \tau
= \frac{\tau}{1-\tau}.
\]

Combining the two displays gives
\[
\bigl\|\Attn(q,K,V)-\Attn_k(q,K,V)\bigr\|_2
= \tau\,\|\mu_{\mathrm{tail}}-\mu_{\mathrm{head}}\|_2
\;\le\; \sqrt{\frac{\tau}{1-\tau}}\,\sqrt{\Var_P(V)}
= \sqrt{D_{\chi^2}(\hat P\Vert P)}\,\sqrt{\Var_P(V)},
\]
as claimed.
\end{proof}

\begin{corollary}[Centered KL--variance bound]
\label{cor:kl-var-appendix}
Using the Top-$k$ identity $\KL(\hat P\Vert P)=-\log(1-\tau)$, one has
\[
\bigl\|\Attn(q,K,V)-\Attn_k(q,K,V)\bigr\|_2
\;\le\; \sqrt{e^{\KL(\hat P\Vert P)}-1}\;\sqrt{\Var_P(V)}.
\]
\end{corollary}

\begin{proof}
For any distributions $P,\hat P$,
$D_{\chi^2}(\hat P\Vert P) \le e^{\KL(\hat P\Vert P)}-1$.
Combining this inequality with Proposition~\ref{prop:chi2-var-appendix}
yields the desired bound.
\end{proof}

\begin{theorem}[Best available certificate]
\label{thm:best-certificate-appendix}
Let $C\ge \max_i \|v_i\|_2$. Then, for Top-$k$ truncation,
\[
\bigl\|\Attn(q,K,V)-\Attn_k(q,K,V)\bigr\|_2 \;\le\;
\min\Bigl\{
\tau\,\mathrm{diam}_{H,T},\;
\sqrt{D_{\chi^2}(\hat P\Vert P)}\,\sqrt{\Var_P(V)},\;
2C\,\tau
\Bigr\}.
\]
\end{theorem}

\begin{proof}
The first term $\tau\,\mathrm{diam}_{H,T}$ is given by
the head--tail diameter bound (Proposition~\ref{prop:head-tail-diam}).
The second term follows from the centered $\chi^2$--variance bound
(Proposition~\ref{prop:chi2-var-appendix}).
Finally, Proposition ~\ref{prop:output} gives the bound $2C\tau$.
Taking the minimum of these three valid upper bounds gives the result.
\end{proof}

\section{Gaussian Model}
\label{app:gaussian-model}

Assume $s_i \sim \mathcal{N}(\mu,\sigma^2)$ i.i.d. Then by the Law of Large Numbers,
\begin{equation}\tag{7}
 \frac{1}{L}\sum_{j=1}^L e^{s_j} \to e^{\mu + \sigma^2/2}.
\end{equation}

\noindent\textbf{Proof of relation (7):}
Let $s_1,s_2,\dots$ be i.i.d.\ $\mathcal N(\mu,\sigma^2)$ and set $X_j \coloneqq e^{s_j}$. We first compute
\[
\EE[X_1]=\EE[e^{s_1}]
=\frac{1}{\sqrt{2\pi}\sigma}\int_{\mathbb R}
\exp\!\left(x-\frac{(x-\mu)^2}{2\sigma^2}\right)\,dx.
\]
We now carry out the completion of the square in detail.
Expand the quadratic term:
\[
x-\frac{(x-\mu)^2}{2\sigma^2}
= x-\frac{x^2-2\mu x+\mu^2}{2\sigma^2}
= -\frac{x^2}{2\sigma^2}+\left(\frac{\mu}{\sigma^2}+1\right)x-\frac{\mu^2}{2\sigma^2}.
\]
Group the $x^2$ and $x$ terms and factor $-\frac{1}{2\sigma^2}$:
\[
x-\frac{(x-\mu)^2}{2\sigma^2}
= -\frac{1}{2\sigma^2}\!\left(x^2-2(\mu+\sigma^2)x+\mu^2\right).
\]
Complete the square inside the brackets:
\[
x^2-2(\mu+\sigma^2)x+\mu^2
= \bigl(x-(\mu+\sigma^2)\bigr)^2-\bigl((\mu+\sigma^2)^2-\mu^2\bigr).
\]
Compute the constant explicitly:
\[
(\mu+\sigma^2)^2-\mu^2=2\mu\sigma^2+\sigma^4.
\]
Therefore
\[
x-\frac{(x-\mu)^2}{2\sigma^2}
= -\frac{(x-(\mu+\sigma^2))^2}{2\sigma^2}+\mu+\frac{\sigma^2}{2}.
\]
Substituting this back into the integral gives
\[
\EE[e^{s_1}]
= \frac{1}{\sqrt{2\pi}\sigma}\int_{\mathbb R}
\exp\!\left(-\frac{(x-(\mu+\sigma^2))^2}{2\sigma^2}\right)\,dx
\;\cdot\; e^{\mu+\sigma^2/2}.
\]
Make the change of variables $z=\frac{x-(\mu+\sigma^2)}{\sigma}$, so $dx=\sigma\,dz$; then
\[
\frac{1}{\sqrt{2\pi}\sigma}\int_{\mathbb R}
\exp\!\left(-\frac{(x-(\mu+\sigma^2))^2}{2\sigma^2}\right)\,dx
= \frac{1}{\sqrt{2\pi}}\int_{\mathbb R} e^{-z^2/2}\,dz
= 1.
\]
Hence
\[
\EE[e^{s_1}] = e^{\mu+\sigma^2/2}.
\]
In particular $\EE[|X_1|]<\infty$, so by the Strong Law of Large Numbers,
\[
\frac{1}{L}\sum_{j=1}^L e^{s_j}
= \frac{1}{L}\sum_{j=1}^L X_j \xrightarrow[L\to\infty]{\text{a.s.}} \EE[X_1]
= e^{\mu+\sigma^2/2},
\]
which is exactly relation (7). \hfill\qedsymbol

\medskip

For a threshold $t$,
\begin{equation}\tag{8}
 \frac{1}{L}\sum_{i:\,s_i>t} e^{s_i}
 \;\longrightarrow\; e^{\mu + \sigma^2/2}\;\Phi_c\!\left(\frac{t-\mu-\sigma^2}{\sigma}\right).
\end{equation}

\noindent\textbf{Proof of relation (8).}
Fix $t\in\R$ and let $s_1,s_2,\dots$ be i.i.d.\ $\mathcal N(\mu,\sigma^2)$.
Define the integrable i.i.d.\ variables
\[
Y_j \coloneqq e^{s_j}\,\mathbf{1}_{\{s_j>t\}}, \qquad j\ge 1.
\]
Then $\sum_{i:\,s_i>t} e^{s_i}=\sum_{j=1}^L Y_j$. By the Strong Law of Large Numbers,
\[
\frac{1}{L}\sum_{j=1}^L Y_j \xrightarrow[L\to\infty]{\text{a.s.}} \EE[Y_1],
\]
so it suffices to compute $\EE[Y_1]=\EE\!\big[e^{s_1}\mathbf{1}_{\{s_1>t\}}\big]$ explicitly.

Let $\varphi_{\mu,\sigma}(x)=\dfrac{1}{\sqrt{2\pi}\,\sigma}\exp\!\left(-\dfrac{(x-\mu)^2}{2\sigma^2}\right)$ be the $\mathcal N(\mu,\sigma^2)$
density. Then
\[
\EE[Y_1]
=\int_{t}^{\infty} e^{x}\,\varphi_{\mu,\sigma}(x)\,\mathrm{d}x
=\frac{1}{\sqrt{2\pi}\,\sigma}\int_{t}^{\infty}
\exp\!\left(x-\frac{(x-\mu)^2}{2\sigma^2}\right)\,\mathrm{d}x.
\]
Using the completed square above,
\[
\EE[Y_1]
= \frac{e^{\mu+\sigma^2/2}}{\sqrt{2\pi}\,\sigma}\int_{t}^{\infty}
\exp\!\left(-\frac{(x-(\mu+\sigma^2))^2}{2\sigma^2}\right)\,\mathrm{d}x.
\]
With the change of variables
$z=\dfrac{x-(\mu+\sigma^2)}{\sigma}$ (so $\mathrm{d}x=\sigma\,\mathrm{d}z$),
the lower limit becomes $z_0=\dfrac{t-\mu-\sigma^2}{\sigma}$ and
\[
\frac{1}{\sqrt{2\pi}\,\sigma}\int_{t}^{\infty}
\exp\!\left(-\frac{(x-(\mu+\sigma^2))^2}{2\sigma^2}\right)\,\mathrm{d}x
= \frac{1}{\sqrt{2\pi}}\int_{z_0}^{\infty} e^{-z^2/2}\,\mathrm{d}z
= \Phi_c\!\left(\frac{t-\mu-\sigma^2}{\sigma}\right).
\]
Therefore
\[
\EE[Y_1]
= e^{\mu+\sigma^2/2}\,\Phi_c\!\left(\frac{t-\mu-\sigma^2}{\sigma}\right),
\]
and combining with the SLLN yields (8). \hfill\qedsymbol

\medskip

Defining
\[
\tau_L(t) := \frac{\sum_{i:\,s_i>t} e^{s_i}}{\sum_{j=1}^L e^{s_j}},
\]
relations (7)–(8) imply
\[
\tau_L(t) \xrightarrow{\ \mathrm{a.s.}\ } \Phi_c\!\Big(\frac{t-\mu-\sigma^2}{\sigma}\Big).
\]
This yields the Gaussian design rule
\[
\alpha_{\mathrm{Gauss}}(\varepsilon;L,\mu,\sigma)\;\approx\;
L\,\Phi_c\!\big(\sigma+\Phi^{-1}(\varepsilon)\big),
\]
as used in the main text.


\end{document}